\documentclass{article}           
\usepackage{iclr2027_conference,times}

\usepackage[utf8]{inputenc}       
\usepackage[T1]{fontenc}          

\AddToHook{cmd/ttfamily/after}{\frenchspacing}

\usepackage{booktabs,multirow,makecell,tabularx} 
\usepackage{graphicx,wrapfig,float,placeins}     
\usepackage{amsfonts,amsmath,amssymb,mathtools,physics} 
\usepackage{subcaption}           
\usepackage{subfiles}             
\usepackage{enumitem}             
\usepackage{ragged2e}             
\usepackage{dsfont}               
\usepackage{nicefrac}             
\usepackage{microtype}            

\usepackage[dvipsnames,table]{xcolor} 
\definecolor{MainRed}{HTML}{9E1B32}
\definecolor{PrimaryRed}{HTML}{C44F53}
\definecolor{LightRed}{HTML}{E8B9C0}
\definecolor{PaleRed}{HTML}{F5DADA}

\usepackage[hyphens]{url}         
\usepackage[hidelinks]{hyperref}  
\hypersetup{
    breaklinks=true,
    colorlinks=true,
    urlcolor=black,
    linkcolor=MainRed,
    citecolor=MainRed,
}
\usepackage{cleveref}             
\crefname{figure}{figure}{figures}
\crefname{Figure}{Figure}{Figures}

\usepackage{amsthm}               

\newtheorem{theorem}{Theorem}

\newtheorem{proposition}{Proposition}

\DeclareMathOperator*{\argmax}{argmax}

\makeatletter
\newcommand{\perf}[1]{%
    \ifdim #1pt > 1.0pt {\setlength{\fboxsep}{2pt}\colorbox{LightRed!70}{\textcolor{PrimaryRed!90!black}{\textbf{#1}}}}
    \else\ifdim #1pt < -1.0pt \textcolor{PrimaryRed!70}{\textbf{--\@gobble#1}}
    \else\ifdim #1pt < 0pt --\@gobble#1%
    \else #1%
    \fi\fi\fi%
}
\makeatother
\newcommand{\highlight}[1]{{\setlength{\fboxsep}{2pt}\colorbox{PaleRed}{\textbf{#1}}}}
\renewcommand{\paragraph}[1]{\textbf{#1} \ \ }

\usepackage{tcolorbox}
\tcbuselibrary{listings,breakable}
\newtcolorbox[
    auto counter,
    list inside=prompt,
    crefname={prompt}{prompts},
]{prompt}[2][]{
    title={Prompt \thetcbcounter: #2},
    colbacktitle=black!60,
    coltitle=white,
    colframe=black!60,
    colback=lightgray!30,
    boxsep=5pt,
    left=3pt,
    right=3pt,
    top=3pt,
    bottom=3pt,
    before skip=15pt,
    after skip=15pt,
    boxrule=0pt,
    breakable,
    #1,
}
\newtcolorbox{takeaway}[1][]{
    colback=PrimaryRed!1,
    colframe=PrimaryRed,
    boxrule=0.6pt,
    arc=0pt,
    left=6pt,
    right=6pt,
    top=3pt,
    bottom=3pt,
    before skip=10pt,
    after skip=10pt,
    #1,
}

\title{Cross-Agent Learning Signals Enable\\Coordinated Role-Decomposed LLM Training}

\author{Jaewan Park$^1$ \quad Solbee Cho$^1$ \quad Jay-Yoon Lee$^{1, \dagger}$\\
$^1$Seoul National University \qquad $^\dagger$Corresponding author\\
\texttt{\{jaejae1112,sbcho0325,lee.jayyoon\}@snu.ac.kr}
}

\iclrfinalcopy 
\begin{document}

\maketitle

\lhead{Preprint}
\vspace{-5pt}

\begin{abstract}
    Agentic search systems must coordinate evidence acquisition and response generation, yet existing approaches either couple both roles under a single agent objective or decompose them without disentangling their respective contributions to the final outcome. We introduce \textbf{DAC (Divide and Cooperate)}, a role-decomposed training framework that, given task-specific external verification signals, trains a searcher and a generator with role-specific cross-verification rewards. DAC allows the generator to abstain when the retrieved evidence appears insufficient, and uses this decision together with externally evaluated search sufficiency to assign appropriate credit to each role. To prevent degenerate over-abstention, we further introduce hard-positive evidence augmentation, which discourages abstaining on sufficient but hard evidence. Across seven general and multi-hop QA benchmarks and two model backbones, DAC consistently outperforms strong single-agent and multi-agent baselines. Controlled evaluations show that these gains arise from coordinated improvements in both search and generation. Our results highlight the importance of explicitly and properly assigning credit across interacting roles when training agentic search systems.
\end{abstract}

\section{Introduction}
\label{sec:intro}
Large language models (LLMs) increasingly solve knowledge-intensive tasks by actively searching for external information rather than relying solely on parametric knowledge. Building on retrieval-augmented generation and ReAct-style tool use \citep{react}, recent work trains such \emph{agentic search} behavior with reinforcement learning (RL), from iterative search-and-reasoning agents \citep{searchr1,r1searcher,research,tips} to long-horizon deep-research systems \citep{sfrdeepresearch,tongyidr}. A recent direction further decomposes these systems into multiple specialized agents \citep{s3,dart,mask}, reflecting the distinct roles involved in acquiring evidence and producing a final response.

However, \emph{\textbf{decomposing the roles does not by itself solve how to train them}}. While role-decomposed and multi-agent frameworks have demonstrated inference-time gains on agentic tasks, reliably training such systems remains challenging due to error propagation and ambiguous credit assignment \citep{sciencescalingagentsystems}. In particular, the final task reward depends jointly on the quality of retrieved evidence and the generator's ability to use it. Assigning this outcome to both agents therefore creates ambiguous credit; a searcher may be penalized despite retrieving sufficient evidence when the generator fails, while a generator may be penalized for an unanswerable input caused by poor search. Likewise, independently rewarding search sufficiency overlooks that what constitutes useful evidence depends on the evolving capabilities of the generator.

\begin{figure}[t]
    \vspace{-5pt}
    \centering
    \includegraphics[width=0.99\textwidth]{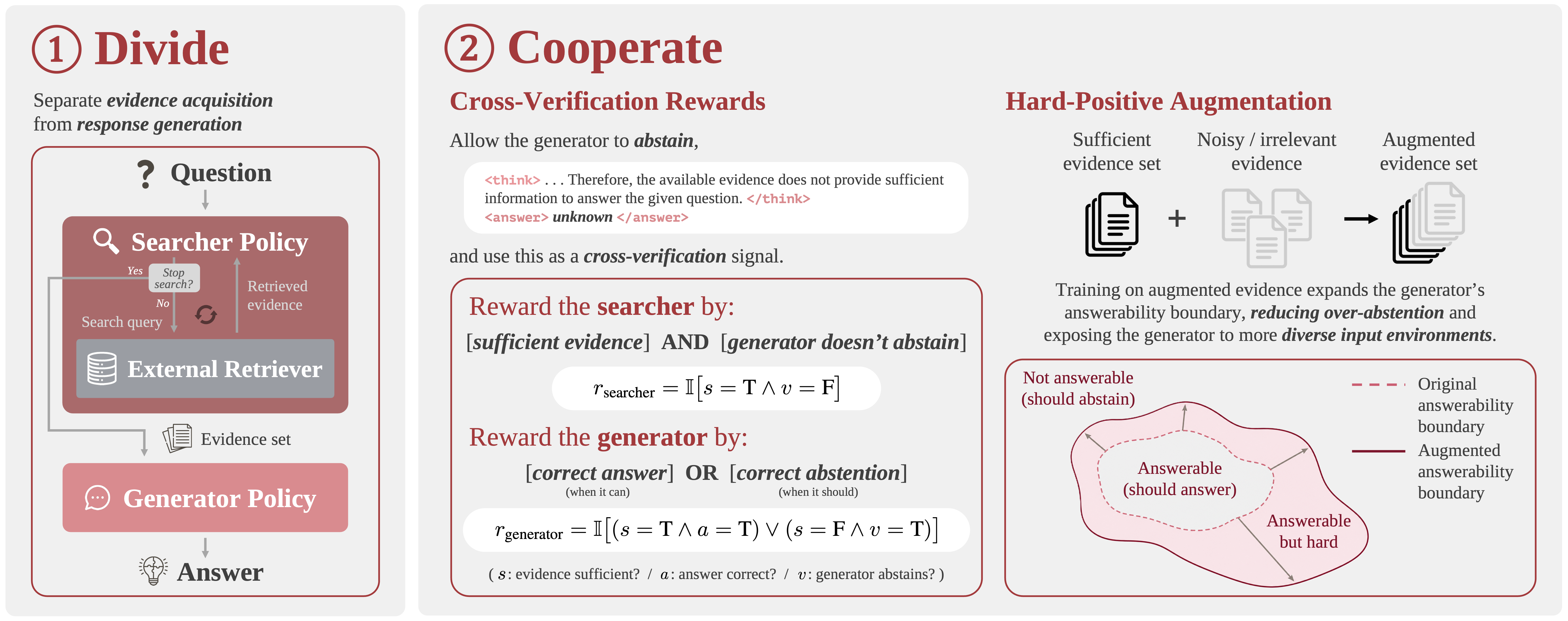}
    \caption{\textbf{Overview of DAC.} DAC decomposes agentic search into a searcher and a generator, and trains each with role-specific rewards. Using cross-verification rewards aligns searcher behavior with downstream generator answerability, while separating credits across agents. Hard-positive augmentation expands the generator's answerability boundary and mitigates over-abstention.}
    \label{fig:overview}
    \vspace{-10pt}
\end{figure}

We introduce \textbf{DAC (Divide and Cooperate)} (\Cref{fig:overview}), a multi-agent training framework that separates evidence acquisition from response generation while coupling their training through \textbf{cross-verification rewards}. DAC gives the generator an additional verification role: it may abstain when the retrieved evidence is insufficient, and this decision is incorporated into the searcher's reward. Consequently, the searcher is only rewarded for evidence that is both sufficient and usable by the current generator, while the generator is rewarded for either answering correctly on sufficient evidence or abstaining when the evidence is insufficient. This aligns their learning, inducing an adaptive curriculum in which search and generation capabilities progressively expand in step with each other, also avoiding penalties caused by failures of the other role.

Yet allowing abstention introduces its own failure mode: the generator may become overly conservative and reject difficult but still answerable evidence. This resembles over-refusal in refusal-aware instruction tuning \citep{craft,grait}, where noisy or imbalanced abstention supervision can bias models toward refusing answerable inputs. We interpret this through a game-theoretic lens, suggesting that without enough hard but answerable training samples, the generator's verification boundary can become overly conservative. This motivates \textbf{hard-positive evidence augmentation}, which adds irrelevant or weakly related evidence to sufficient contexts, discouraging abstention while promoting correct answers despite the added noise.

We evaluate DAC on seven general and multi-hop question answering (QA) benchmarks. DAC consistently outperforms training-free, single-agent training, and prior multi-agent training approaches. Controlled experiments further disentangle these gains, showing that DAC improves in both search and generation separately. Our analyses show that cross-verification produces more balanced learning dynamics than using external verification signals alone. Furthermore, learned abstention provides a useful inference-time signal: inputs on which the generator abstain have substantially lower answer accuracy, allowing abstention to trigger further search or request user assistance when additional evidence is needed, despite the searcher's sufficiency decision.

Our main contributions are:
\begin{itemize}[leftmargin=1.5em,itemsep=0pt,topsep=0pt]
\item We propose \textbf{DAC}, a cooperative training framework for role-decomposed agentic search. Given external search and answer verification, DAC combines these signals with the generator's evolving abstention decision to construct role-specific \textbf{cross-verification rewards}. This provides an adaptive curriculum for joint training, matching search capability to the generator's evolving answering capability while disentangling credit between search and generation.
\item We analyze the training dynamics of this cooperative interaction through a theoretical lens, identifying over-abstention as a key failure mode and introducing \textbf{hard-positive evidence augmentation} to encourage the generator to accept and correctly answer on harder sufficient evidence.
\item We evaluate DAC across \textbf{seven QA benchmarks} with two LLM backbones. DAC consistently outperforms strong single-agent and prior multi-agent baselines, while controlled analysis shows that the gains come from improvements in \textbf{both search quality and answer generation}.
\end{itemize}

\section{Problem setting: Agentic Search}
\label{sec:problem_setting}
We define \textbf{agentic search} as a broad class of tasks in which an LLM agent first acquires task-relevant evidence and then uses that evidence to produce a final response. This differs from interaction-centric tasks where the primary challenge is acting within an external environment and the final response mainly summarizes the completed actions. We decompose agentic search into two sequentially interacting roles, \textbf{evidence acquisition} and \textbf{response generation}. These roles may be implemented using either a shared policy or separately parameterized policies.

Under this decomposition, the agentic task proceeds as follows:
\begin{takeaway}
\begin{enumerate}[leftmargin=*,itemsep=1pt]
    \item A task specification $Q$, also called a \textbf{question}, is provided to the system. The system consists of two agents: a \textbf{searcher agent} $\mathcal{S}$ and a \textbf{generator agent} $\mathcal{G}$.

    \item $\mathcal{S}$ performs iterative evidence acquisition. At each iteration $t$, it observes $Q$ and the \textbf{evidence} collected through iteration $t-1$, $\mathcal{E}_{t-1} = \bigcup_{i<t} e_i$, with initial condition $t=1$ and $\mathcal{E}_{0}=\varnothing$.

    Based on this information, $\mathcal{S}$ either \textcolor[HTML]{9E4A52}{\textbf{\texttt{[}continues the search loop\texttt{]}}} or, if it determines that the collected evidence is sufficient, stops the search loop and sets the final accumulated evidence set $\mathcal{E} = \mathcal{E}_{t-1}$ before \textcolor[HTML]{C07A80}{\textbf{\texttt{[}passing evidence to the generator\texttt{]}}}.

    \item \textcolor[HTML]{9E4A52}{\textbf{\texttt{[}Continue Searching\texttt{]}}} $\mathcal{S}$ generates a query $q_t$ and pass it to an external tool, which returns the search results that $\mathcal{S}$ processes into step-level evidence $e_t$. This evidence is accumulated into the evidence set, $\mathcal{E}_t = \mathcal{E}_{t-1} \cup e_t$. The process then returns to step 2 with $t \leftarrow t+1$.

    \item \textcolor[HTML]{C07A80}{\textbf{\texttt{[}Pass Evidence to the Generator\texttt{]}}} $\mathcal{G}$ receives the original question $Q$ together with the final accumulated evidence set $\mathcal{E}$ and produces the final \textbf{answer} $\mathcal{A} = \mathcal{G}(Q, \mathcal{E})$.\end{enumerate}
\end{takeaway}

We deliberately define \textbf{evidence} broadly as the accumulated task-relevant context rather than restricting it to raw tool outputs. Thus, $\mathcal{E}_t$ may include extracted facts, summaries, intermediate reasoning, or other processed information that the searcher and generator can utilize, depending on the task.

\section{DAC}
\label{sec:methodology}
We introduce \textbf{DAC (Divide and Cooperate)}, a multi-agent training framework that optimizes evidence acquisition and response generation separately. DAC enables joint RL training with role-specific rewards that better address credit assignment between search and generation.

\subsection{Reward design for joint optimization}
\label{sec:reward_design}

Under the two-stage formulation, we aim to jointly optimize $\mathcal{S}$ and $\mathcal{G}$ using RL. The final objective of agentic search is the quality of the final response. For example, in question answering, this can be instantiated as the exact match (EM) score of the generated answer, as in \citet{searchr1}. 

A straightforward design is to assign this \emph{answer-level reward} to both agents, but this creates a credit assignment problem, when one agent is penalized for the other's failure. An alternative is to reward the searcher directly using an external \emph{evidence-sufficiency signal}. However, this signal alone does not indicate whether the evolving generator can utilize the retrieved evidence well. Directly optimizing external sufficiency may therefore improve retrieval faster than the generator's ability to exploit it. Moreover, it does not address how to avoid penalizing the generator when retrieval is insufficient.

To address this issue, we employ \highlight{cross-verification rewards}, by giving the generator an additional verification role. The generator is allowed to \emph{abstain} when it judges that the provided evidence is insufficient. This abstention decision is then incorporated into the searcher's reward. At the same time, the generator is rewarded not only for answering correctly, but also for correctly abstaining when the evidence is indeed insufficient according to an external verifier.
More formally, let
\begin{align*}
    s\qty(Q, \mathcal{E}) &= \mathbb{I}\bigl[\text{evidence $\mathcal{E}$ is sufficient to answer question $Q$}\bigr], \\
    v\qty(\mathcal{A}) &= \mathbb{I}\bigl[\text{generation $\mathcal{A}$ is an abstention}\bigr], \\
    a\qty(Q, \mathcal{A}) &= \mathbb{I}\bigl[\text{generation $\mathcal{A}$ is a qualified answer to question $Q$}\bigr],
\end{align*}
where $s$ and $a$ are task-specific signals judged by an external verifier.

Then, we may define the searcher and generator rewards as:
\begin{align}
    r_{\mathcal{S}}\qty(Q, \mathcal{E}, \mathcal{A}) &= s\qty(Q, \mathcal{E})\bigl(1 - v\qty(\mathcal{A})\bigr), \\
    r_{\mathcal{G}}\qty(Q, \mathcal{E}, \mathcal{A})
    &= s\qty(Q, \mathcal{E})a\qty(Q, \mathcal{A}) + \bigl(1 - s\qty(Q, \mathcal{E})\bigr)v\qty(\mathcal{A})
    = \begin{cases}
        a\qty(Q, \mathcal{A}), & \text{if } s\qty(Q, \mathcal{E}) = 1, \\[3pt]
        v\qty(\mathcal{A}), & \text{if } s\qty(Q, \mathcal{E}) = 0.
    \end{cases}
\end{align}
Thus, the searcher is rewarded for retrieving evidence that is both externally sufficient and deemed sufficient by the generator, avoiding penalties due to generator-side answer failures. The generator is rewarded for either answering correctly on sufficient evidence or abstaining on insufficient evidence.

Both agents ultimately optimize the external verification criteria, regardless of whether cross-verification is used (see \cref{app:theoretical_analysis}). Cross-verification mainly shapes and aligns how they reach this solution. For the searcher, it favors sufficient search modes that the current generator can likely use, providing a more focused exploration path through the multimodal exploration space. For the generator, the searcher defines its evolving RL environment; through abstention, the generator steers this environment toward evidence better matched to its current capability, making each intermediate learning problem easier. This interaction induces a mutually adaptive curriculum for both agents.

\subsection{Mitigating over-abstention with hard-positive evidence augmentation}
\label{sec:augmentation}

A predictable failure mode of our framework is \emph{over-abstention}. Because abstaining on insufficient evidence provides a relatively direct reward signal, whereas obtaining reward on sufficient evidence requires the generator to both accept the input and answer correctly, training can become biased toward a conservative abstention behavior when hard-positive examples are scarce. This can hurt even when abstention is disabled at test time; if the generator was too heavily biased toward abstention during training, it may not learn to answer well from hard, noisy, or imperfect evidence.

To mitigate this, we must pressure the generator to expand its answerable region rather than accepting only clean sufficient evidence. We therefore augment generator training with \highlight{hard-positive evidence} samples. For each sufficient pair $(Q,\mathcal{E})$ with $s(Q,\mathcal{E})=1$, we add irrelevant or weakly related evidence $\mathcal{E}_N$ to form $\mathcal{E}'=\mathcal{E}\cup\mathcal{E}_N$, as illustrated in \Cref{fig:overview}. Since $\mathcal{E}$ remains sufficient, this creates harder positive examples that encourage the generator to expand its answerable region rather than abstaining on non-ideal evidence. This mitigation is also motivated by a game-theoretic intuition. We provide a theoretical interpretation of this effect in \cref{sec:theoretical_analysis}.

\subsection{Turn-level difference rewards for searcher training}

We also use temporal-difference turn-level rewards, following prior work \citep{tips,oases,mask}. Since the final search reward provides only trajectory-level feedback, we obtain a denser signal by evaluating the reward change after each search turn.
When $\mathcal{E}_t$ denotes the evidence collected up to turn $t$, we run the generator on each intermediate evidence state and define the turn-level reward as the difference
\begin{equation*}
    s_{\mathcal{S}}^{(t)} = r_{\mathcal{S}}\bigl(Q, \mathcal{E}_t, \mathcal{G}(Q,\mathcal{E}_t)\bigr), \qquad r_{\mathcal{S}}^{(t)} = s_{\mathcal{S}}^{(t)} - s_{\mathcal{S}}^{(t-1)}.
\end{equation*}
Thus, turns are rewarded according to how they change the usefulness of the accumulated evidence. While prior methods rely on potentially unreliable metrics such as model perplexity to estimate $s_{\mathcal{S}}^{(t)}$, we benefit from using a decomposed design as it directly enables such intermediate-state feedback.

\section{Experiments}
\label{sec:experiments}
We experiment on knowledge-based QA tasks. The system is equipped with a dense-retrieval system as the external tool. Full implementation details are provided in Appendix~\ref{app:further_implementation_details}.

\begin{table}[tb]
\centering
\caption{\textbf{Exact match (EM) on QA benchmarks.} Datasets marked with $^\ast$ are in-domain datasets. \textbf{Bold} and \underline{underlined} numbers indicate the best and second best results for each model, respectively.
}
\label{tab:main_results}
\setlength{\tabcolsep}{4pt}
\resizebox{\textwidth}{!}{%
\begin{tabular}{llcccccccc}
\toprule
\multirow{2.5}{*}{Type} & \multirow{2.5}{*}{Method}
& \multicolumn{3}{c}{General QA} & \multicolumn{4}{c}{Multi-Hop QA} & \multirow{2.5}{*}{Average} \\
\cmidrule(lr){3-5} \cmidrule(lr){6-9}
& & NQ{$^*$} & TriviaQA & PopQA
& HotpotQA\!{$^*$} & 2Wiki & MuSiQue & Bamb.
& \\

\specialrule{\lightrulewidth}{\aboverulesep}{0pt}
\rowcolor[HTML]{F9ECEC}
\multicolumn{10}{c}{\rule[-1ex]{0pt}{3.4ex}\texttt{Qwen2.5-7B-Instruct}} \\
\specialrule{\lightrulewidth}{0pt}{\belowrulesep}
\multirow{5}{*}{\makecell[l]{Training- \\ free}} & Direct & 18.2 & 44.8 & 16.5 & 19.1 & 24.6 & 4.3 & 9.6 & 19.6 \\
& CoT & 19.3 & 47.0 & 16.0 & 20.8 & 21.2 & 5.3 & 25.6 & 22.2 \\
& Naive RAG & 36.9 & 58.9 & 38.9 & 29.3 & 17.9 & 6.4 & 24.0 & 30.3 \\
& IRCoT & 16.6 & 29.8 & 20.7 & 18.7 & 19.5 & 4.8 & 15.2 & 17.9 \\
& Search-o1 & 31.3 & 52.3 & 33.6 & 29.0 & 26.0 & 11.6 & 33.6 & 31.1 \\
\midrule
\multirow{2}{*}{Single-agent} & Search-R1 & \textbf{43.9} & 62.0 & \underline{42.6} & 38.5 & 34.8 & 15.6 & 40.0 & 39.6 \\
& TIPS & 39.8 & \underline{62.2} & 42.2 & 37.9 & 36.6 & 13.4 & 35.2 & 38.2 \\
\midrule
\multirow{3}{*}{Multi-agent} & S3 & 7.3 & 36.0 & 19.0 & 17.3 & 12.6 & 4.0 & 1.6 & 14.0 \\
& DART & \underline{42.4} & 61.6 & 41.0 & \underline{39.9} & \textbf{40.6} & \underline{15.7} & \underline{40.8} & \underline{40.3} \\
& M-ASK & 33.8 & 60.1 & 36.8 & 37.4 & 31.8 & 12.1 & 33.6 & 35.1 \\
\midrule
Multi-agent & \textbf{DAC} & \textbf{43.9} & \textbf{63.4} & \textbf{43.8} & \textbf{43.2} & \underline{39.8} & \textbf{19.0} & \textbf{43.2} & \textbf{42.3} \\

\specialrule{\lightrulewidth}{\aboverulesep}{0pt}
\rowcolor[HTML]{F9ECEC}
\multicolumn{10}{c}{\rule[-1ex]{0pt}{3.4ex}\texttt{Qwen3-8B (Non-Thinking)}} \\
\specialrule{\lightrulewidth}{0pt}{\belowrulesep}
\multirow{5}{*}{\makecell[l]{Training- \\ free}} & Direct & 16.3 & 42.4 & 15.4 & 18.2 & 25.2 & 3.4 & 10.4 & 18.8 \\
& CoT & 19.9 & 45.3 & 15.9 & 21.3 & 25.5 & 5.8 & 32.8 & 23.8 \\
& Naive RAG & 35.9 & 58.5 & 38.7 & 29.0 & 19.3 & 7.0 & 26.4 & 30.7 \\
& IRCoT & 35.6 & 57.9 & 36.3 & 36.9 & 26.3 & 12.0 & 36.8 & 34.4 \\
& Search-o1 & 32.8 & 49.9 & 35.7 & 29.9 & 31.7 & 11.3 & 32.0 & 31.9 \\
\midrule
\multirow{2}{*}{Single-agent} & Search-R1 & \underline{48.1} & 66.1 & 45.3 & 45.0 & 43.9 & \underline{21.7} & \underline{48.8} & \underline{45.6} \\
& TIPS & 39.7 & 62.5 & 41.9 & 38.3 & 41.6 & 16.9 & 44.8 & 40.9 \\ 
\midrule
\multirow{3}{*}{Multi-agent} & S3 & 12.2 & 47.4 & 29.5 & 22.9 & 16.7 & 3.8 & 17.6 & 21.4 \\
& DART & 46.8 & \underline{66.3} & \underline{47.3} & \underline{45.3} & \underline{44.5} & 20.9 & 48.0 & \underline{45.6} \\
& M-ASK & 33.5 & 64.1 & 38.0 & 41.8 & 37.4 & 13.8 & 36.0 & 37.8 \\
\midrule
Multi-agent & \textbf{DAC} & \textbf{48.4} & \textbf{67.2} & \textbf{47.9} & \textbf{46.6} & \textbf{48.7} & \textbf{22.3} & \textbf{49.6} & \textbf{47.2} \\
\bottomrule
\end{tabular}%
}
\end{table}

\paragraph{Trajectory management}
We maintain evidence as the unprocessed searcher interaction trajectory; following \citet{searchr1}, this is serialized as repeated \texttt{<think>...</think>}, \texttt{<search>...</search>}, and \texttt{<information>...</information>} segments, containing reasoning content, search queries, and retrieved documents, respectively. The loop terminates when the searcher emits \texttt{<stop>} after \texttt{</think>}. During training, \texttt{information} blocks are treated as environment outputs and masked out. For the generator, we pass only the retrieved documents as evidence, which keeps the input format simple and makes noise-based evidence augmentation straightforward.

\begin{figure}[tb]
    \vspace{-5pt}
    \centering
    \includegraphics[width=0.99\textwidth]{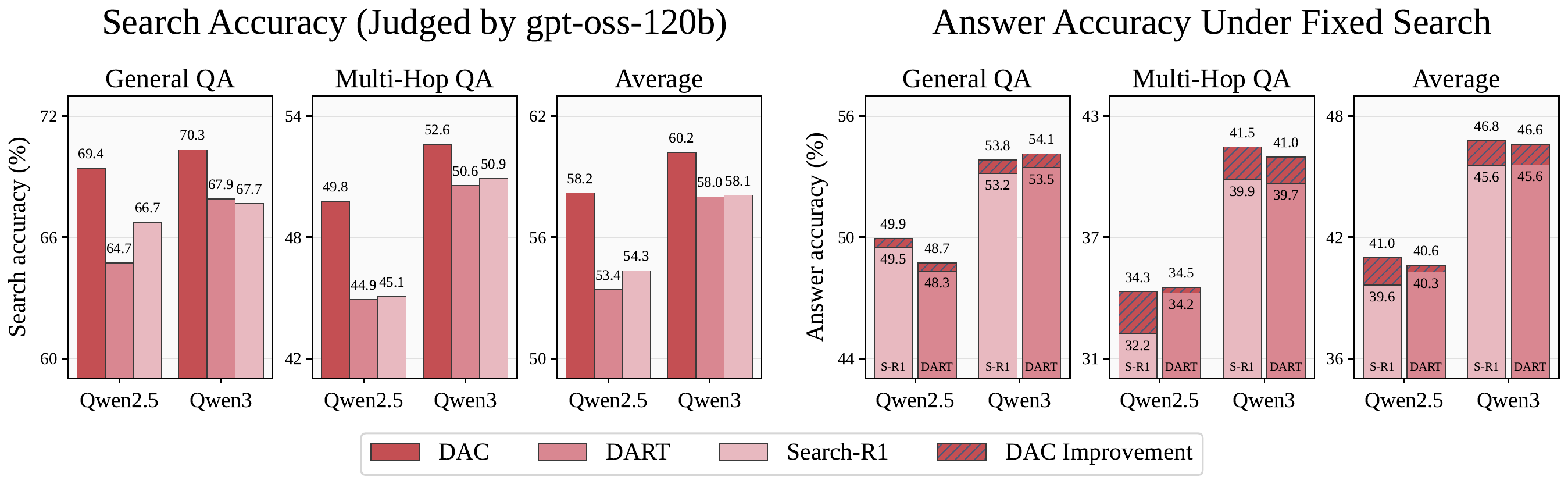}
    \caption{\textbf{Disentangling gains in search and generation.} \textit{Left}: search accuracy of the retrieved evidence, judged by \texttt{gpt-oss-120b}. \textit{Right}: answer EM when DAC generates from the same search results each produced by Search-R1 or DART, respectively. Hatched regions indicate DAC's improvement over each corresponding baseline.}
    \label{fig:disentangled_results}
\end{figure}

\paragraph{Verifier instantiation}
DAC does not prescribe a particular implementation of the external verification signals. In our training runs, we use EM-based verification for both answer correctness and search sufficiency. Specifically, for search, we use substring EM \citep{longrag,searchr1_2}, where the search result is considered sufficient if at least one document contains or the generated answer matches the gold answer. This is a lightweight choice for computational efficiency. Other semantic, learned, or task-specific verifiers can be substituted without changing the DAC objective.

\paragraph{Training}
We conduct experiments with two backbone models: \texttt{Qwen2.5-7B-Instruct} \citep{qwen25} and \texttt{Qwen3-8B (Non-Thinking)} \citep{qwen3}. For each backbone, the searcher and generator are implemented using a single shared model. Their RL losses are computed separately and summed into a single objective used to update the shared parameters. We merge the training splits of NQ \citep{nq} and HotpotQA \citep{hotpotqa} to construct the training dataset.

\paragraph{Benchmarks}
We evaluate on seven benchmarks. For general QA, we use NQ, TriviaQA \citep{triviaqa}, and PopQA \citep{popqa}. For multi-hop QA, we use HotpotQA, 2WikiMultihopQA \citep{2wiki}, MuSiQue \citep{musique}, and Bamboogle \citep{selfask}. During DAC evaluation, generator abstention is disabled.

\paragraph{Baselines}
We compare against three groups of baselines: \textbf{(1) training-free methods}, including direct answering and chain-of-thought (CoT) prompting \citep{cot}, naive RAG \citep{rag}, IRCoT \citep{ircot}, and Search-o1 \citep{searcho1}; \textbf{(2) single-agent training methods}, including Search-R1 \citep{searchr1} and TIPS \citep{tips}; \textbf{(3) multi-agent training methods}, including S3 \citep{s3}, DART \citep{dart}, and M-ASK \citep{mask}. Detailed descriptions and reproduction procedures for each are provided in Appendix~\ref{app:baselines}.

\section{Main results}
\label{sec:results}
\Cref{tab:main_results} presents the main results across seven QA benchmarks, with complementary LLM-judge evaluations in Appendix~\ref{app:llm_judge_results}. DAC consistently achieves the best overall performance on both model backbones. In particular, it outperforms naive single-agent optimization (Search-R1), as well as single-agent training augmented with turn-level rewards (TIPS), suggesting that denser feedback alone is insufficient. DAC also surpasses prior multi-agent approaches (S3, DART, and M-ASK), which rely primarily on final answer-correctness signals. Together, these results support DAC's use of role-specific, cross-verification rewards to better assign credit in decomposed agentic search.

To understand the source of these gains, we compare DAC with Search-R1 and DART, the strongest overall baselines, by disentangling search quality from answer generation. DAC's searcher outperforms both methods (\Cref{fig:disentangled_results} left), and when given the same retrieved evidence as each baseline, DAC's generator still achieves higher answer accuracy (\Cref{fig:disentangled_results} right). This shows that DAC improves both \emph{search quality} and the \emph{generator's ability to use the retrieved evidence}.

\section{Analysis}
\label{sec:analysis}
\subsection{Ablation study I: Effect of cross-verification rewards}
\label{sec:ablation_cross_verification}

\begin{figure}[tb]
    \centering
    \includegraphics[width=0.99\textwidth]{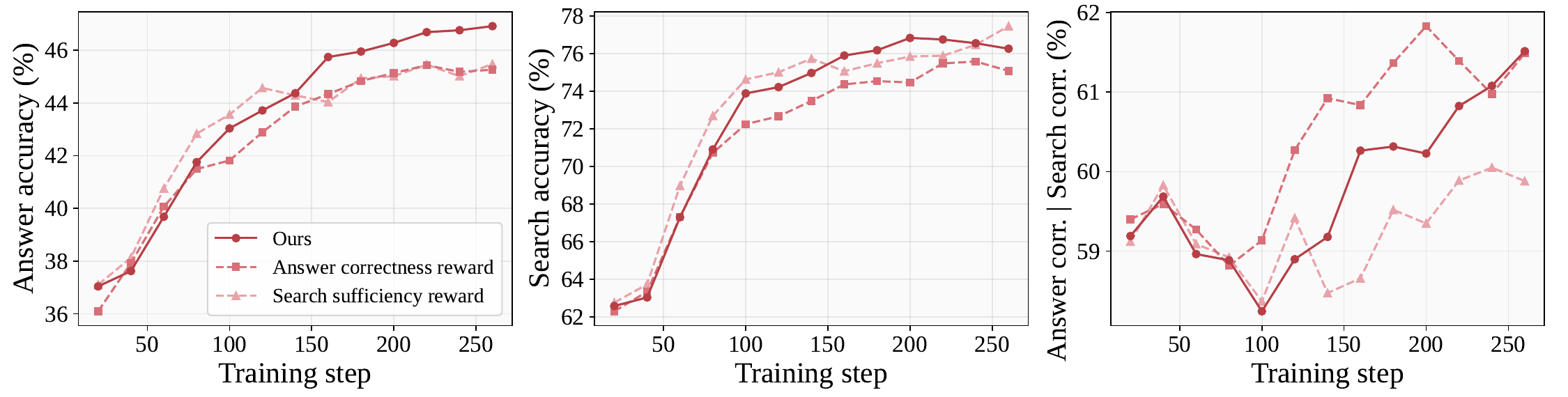}
    \caption{\textbf{Training dynamics under different searcher rewards.} Validation performance on NQ and HotpotQA. \emph{Left}: answer EM. \emph{Center} substring EM. \emph{Right}: answer EM on correct search.}
    \label{fig:searcher_reward_ablation}
\end{figure}

We first examine the role of cross-verification by comparing DAC with two simpler searcher rewards discussed in \cref{sec:reward_design}: \emph{using only external signals}, answer correctness or search sufficiency, without the generator's verification decision.
As shown in \Cref{tab:ablation}, both variants underperform DAC.
\Cref{fig:searcher_reward_ablation} further reveals different training dynamics. With answer correctness as the searcher reward, conditional answer accuracy remains high, but search accuracy improves more slowly and overall answer accuracy consistently trails DAC. In contrast, the search sufficiency reward rapidly improves search, while conditional answer accuracy improves little, indicating that search advances faster than the generator's ability to exploit the retrieved evidence. DAC improves both components more coherently, suggesting that cross-verification helps align their training trajectories.

We further analyze this interaction through cross-checkpoint searcher--generator evaluations in Appendix~\ref{app:ablation_cross_verification_further}. The results suggest that DAC induces a gradual curriculum: as the generator improves, the searcher expands successful retrieval toward previously unresolved, harder questions, while preserving or improving evidence quality on questions that were already retrievable.

\begin{takeaway}
\textbf{Takeaway 1.} \ Cross-verification induces an adaptive curriculum in which search progressively expands with the generator's evolving capability, leading to better final performance.
\end{takeaway}

\begin{table}[tb]
\centering
\caption{\textbf{Test results for training variants on \texttt{Qwen3-8B}.} Scores are reported as differences relative to DAC. {\setlength{\fboxsep}{1pt}\colorbox{LightRed!70}{\textcolor{PrimaryRed!90!black}{\textbf{\,Highlighted text\,}}}} marks a difference greater than +1, and \textcolor{PrimaryRed!70}{\textbf{colored text}} marks less than -1.}
\label{tab:ablation}
\setlength{\tabcolsep}{4pt}
\resizebox{\textwidth}{!}{%
\begin{tabular}{lcccccccc}
\toprule
\multirow{2.5}{*}{Method} & \multicolumn{3}{c}{General QA} & \multicolumn{4}{c}{Multi-Hop QA} & \multirow{2.5}{*}{Average} \\
\cmidrule(lr){2-4} \cmidrule(lr){5-8} & NQ{$^*$} & TriviaQA & PopQA & HotpotQA\!{$^*$} & 2Wiki & MuSiQue & Bamb. & \\
\specialrule{\lightrulewidth}{\aboverulesep}{0pt}
\rowcolor{gray!10}
\rule[-1ex]{0pt}{3.5ex} \textbf{DAC} (Base) & 48.4 & 67.2 & 47.9 & 46.6 & 48.7 & 22.3 & 49.6 & 47.2 \\
\specialrule{\lightrulewidth}{0pt}{\belowrulesep}
Answer correctness as searcher reward & \perf{+0.2} & \perf{-1.2} & \perf{-1.1} & \perf{-1.6} & \perf{-3.4} & \perf{-1.7} & \perf{-2.4} & \perf{-1.6} \\
Search sufficiency as searcher reward & \perf{-0.3} & \perf{-1.7} & \perf{-1.0} & \perf{-1.0} & \perf{-1.1} & \perf{-1.9} & \perf{-8.8} & \perf{-2.2} \\
\midrule
w/o augmented samples & \perf{-3.4} & \perf{-1.5} & \perf{-2.8} & \perf{-2.4} & \perf{-5.3} & \perf{-1.9} & \perf{-0.8} & \perf{-2.6} \\
\midrule
Non-abstention additional training & \perf{+1.8} & \perf{-0.5} & \perf{+0.9} & \perf{+0.2} & \perf{+1.1} & \perf{+0.1} & \perf{+0.8} & \perf{+0.7} \\
\bottomrule
\end{tabular}%
}
\end{table}

\subsection{Ablation study II: Effect of hard-positive evidence augmentation}
\label{sec:ablation_augmentation}

\begin{wrapfigure}{r}{0.4\textwidth}
    \vspace{-15pt}
    \centering
    \includegraphics[width=0.38\textwidth]{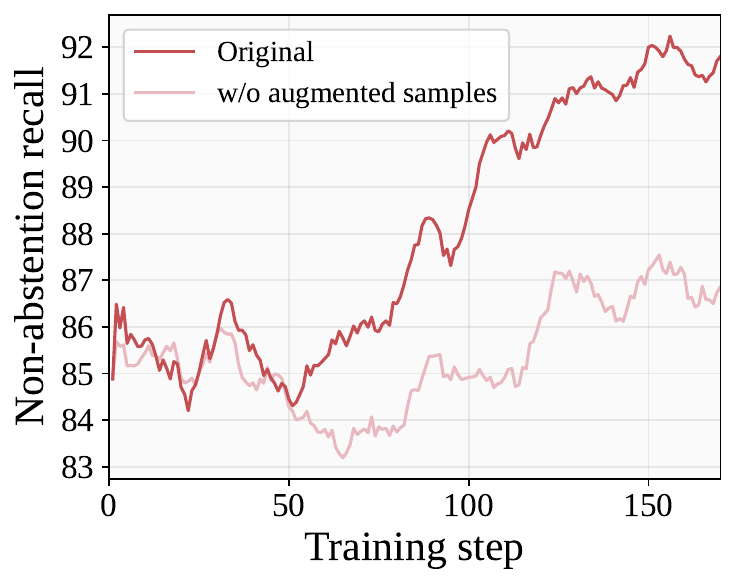}
    \captionsetup{font=small}
    \caption{\textbf{Non-abstention recall of the generator.} We measure during training how often the generator correctly chooses to answer rather than abstain.}
    \label{fig:no_noise}
    \vspace{-15pt}
\end{wrapfigure}

We ablate hard-positive evidence augmentation by removing the noise documents added to sufficient evidence examples. As shown in \Cref{tab:ablation}, this greatly lowers overall test performance. The training dynamics in \Cref{fig:no_noise} clarify the source of this difference. Without augmented samples, the generator has substantially lower recall for non-abstention, i.e., it answers less often in cases where it should answer. This is precisely the over-abstention failure mode discussed in \cref{sec:augmentation}: the generator learns an overly strict verification boundary and rejects sufficient but imperfect evidence. Hard-positive augmentation mitigates this by exposing the generator to sufficient but noisy evidence, encouraging it to accept and answer correctly under noisy evidence rather than abstain too conservatively. We also analyze this from a theoretical perspective in \cref{sec:theoretical_analysis}.

\begin{takeaway}
\textbf{Takeaway 2.} \ Hard-positive evidence augmentation prevents over-abstention during training, improving final performance even when abstention is disabled at test time.
\end{takeaway}

\subsection{Abstention at inference time}
\label{sec:abstention}

\begin{figure}[tb]
    \vspace{-5pt}
    \centering
    \includegraphics[width=0.99\textwidth]{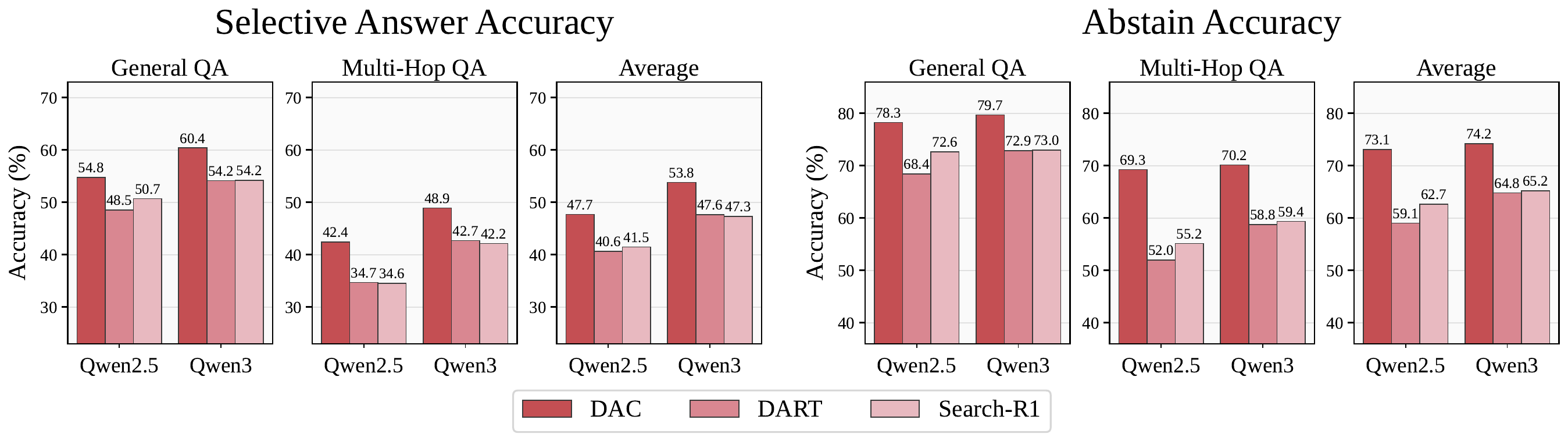}
    \caption{\textbf{Selective and abstain accuracy.} \textit{Left}: selective accuracy measures answer EM among non-abstained examples. \textit{Right}: abstain accuracy measures whether the model correctly decides to answer or abstain based on search accuracy, which is judged by \texttt{gpt-oss-120b}.}
    \label{fig:abstention}
\end{figure}

Abstention behavior learned by DAC is meaningful. In Appendix~\ref{app:abstain_forced}, we find that when the generator is forced to answer on inputs it has abstained on, accuracy drops substantially, indicating that the generator recognizes when evidence is insufficient rather than hastily providing an inaccurate answer.

Ideally, DAC optimizes the searcher toward always retrieving sufficient evidence, in which case the generator would no longer need to abstain. In practice, however, search failures remain, making abstention a useful inference-time capability. As \Cref{fig:abstention} shows, DAC learns substantially better abstention behavior than Search-R1 and DART, while also maintaining higher selective answer accuracy. For the baselines, we treat continuing to search until the maximum turn limit as abstention. Hence, beyond its role in training, abstention could also be exposed as a standalone functionality in deployed systems. For instance, when the generator abstains, the system could automatically trigger additional search rounds to collect more evidence before producing a final answer.

\begin{wrapfigure}{r}{0.4\textwidth}
    \vspace{-10pt}
    \centering
    \includegraphics[width=0.38\textwidth]{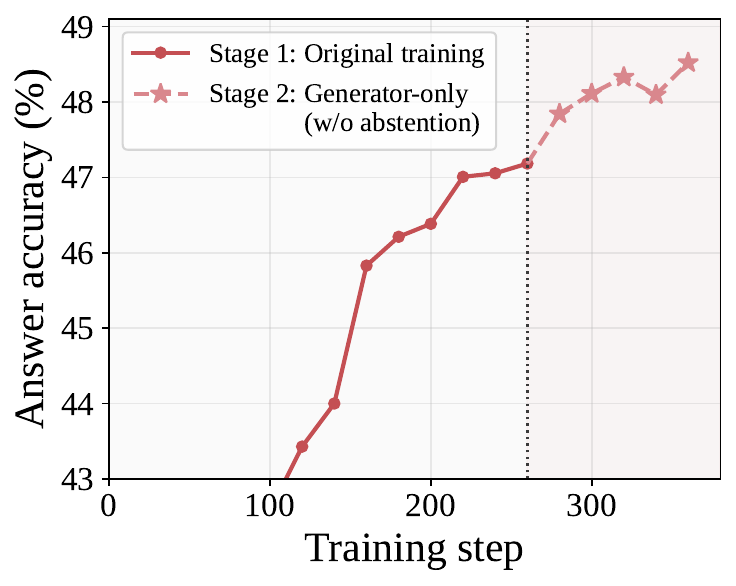}
    \captionsetup{font=small}
    \caption{\textbf{Generator-only post-training without abstention.} Validation answer EM on NQ and HotpotQA, during joint training and the subsequent generator-only stage with abstention disabled.}
    \label{fig:answer_only}
    \vspace{-10pt}
\end{wrapfigure}

Yet, for direct comparison with the baselines under standard QA evaluation, abstention is disabled in our main evaluation. Despite being trained with an abstention option, the model readily follows the test-time instruction to always answer, and this simple strategy already yields strong performance. Nevertheless, if the resulting train--test mismatch is a concern, it can be further mitigated with a short post-training stage. Specifically, after the original joint training finishes, or once the searcher reward has largely plateaued, we continue training the shared policy with only the generator objective, while disabling abstention reward and only rewarding answer correctness. This preserves the learned search behavior while adapting the generator to the inference-time setting. On \texttt{Qwen3-8B}, roughly 100 additional training steps yield a further improvement in answer accuracy (see \Cref{fig:answer_only} and the ``Non-abstention additional training'' row of \Cref{tab:ablation}), suggesting that only lightweight adaptation is needed.

\subsection{Using dedicated policies for each agent}
\label{sec:shared_dedicated}

\begin{wrapfigure}{r}{0.4\textwidth}
    \vspace{-15pt}
    \centering
    \includegraphics[width=0.38\textwidth]{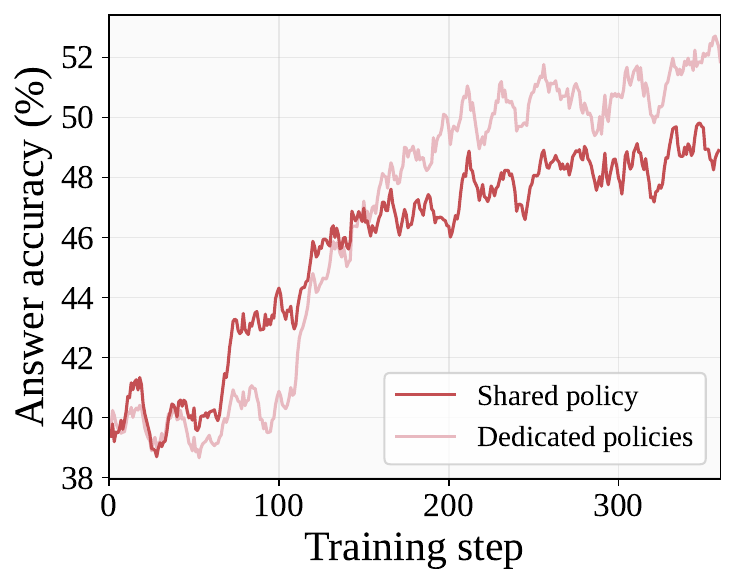}
    \captionsetup{font=small}
    \caption{\textbf{Shared vs. dedicated LoRA policies.} Training answer EM with a shared policy or dedicated policies for the searcher and generator. Curves show a moving average over 10 training steps. Dedicated policies overtake the shared policy as training progresses.}
    \label{fig:shared_dedicated}
    \vspace{-20pt}
\end{wrapfigure}

For computational efficiency, our main experiments optimize a single shared model for both the searcher and generator. To isolate unintended effects of parameter sharing, on \texttt{Qwen3-8B}, we additionally compare a shared LoRA policy against dedicated LoRA policies for the two roles. Full results are reported in Appendix~\ref{app:shared_dedicated_results} and training dynamics are shown in \Cref{fig:shared_dedicated}. Dedicated policies achieve better final performance overall.

Interestingly, the shared policy performs slightly better early in training, suggesting that parameter sharing may initially facilitate transfer between the two roles. As training progresses, however, the dedicated policies overtake the shared model and maintain a growing advantage, as the searcher and generator become more specialized and interference between their distinct objectives is reduced. These results suggest that the benefits of role specialization may extend to independently full-parameter fine-tuned policies, which we leave for future work.

\subsection{Theoretical analysis of DAC}
\label{sec:theoretical_analysis}

DAC's cross-verification relies on the generator accepting sufficient evidence so that retrieval useful to the generator receives a positive searcher reward. When sufficient training examples are concentrated on easy cases, the generator may instead favor a conservative acceptance rule that rejects harder but answerable evidence. Such over-abstention suppresses the searcher's reward on these inputs. Hard-positive augmentation supplies the corrective signal needed to counter this failure mode.

We study a threshold model in which the generator answers on an input evidence if and only if $h \leq \theta$, where $h \in [0, 1]$ indicates the evidence's scalar \textbf{difficulty score} and $\theta \in [0, 1]$ is the generator's chosen \textbf{acceptance threshold}. Let $\gamma$ denote the training relative mixture weight of augmented sufficient evidence in the generator's training input distribution. Holding the evidence distributions and conditional answering ability fixed, we obtain the following result. Appendix~\ref{app:theoretical_analysis} gives the full formulation and proof.

\begin{proposition}[Conservative regime and its removal]
\label{prop:dac_regime_change}
Under the model assumptions in Appendix~\ref{app:dac_augmentation_setup}, suppose ordinary sufficient evidence has difficulty at most $b < 1$, and any threshold above $b$ causes the generator to accept additional insufficient evidence it would have abstained on at threshold $b$. If augmented evidence has positive mass above $b$ on which the generator can answer correctly with positive probability, then there exists a finite $\Gamma > 0$ such that every optimal threshold $\theta^\star_\gamma$ satisfies
\begin{equation}
    \theta^\star_\gamma \leq b \quad \text{if} \;\; \gamma = 0, \qquad \theta^\star_\gamma > b \quad \text{if} \;\; \gamma > \Gamma.
\end{equation}
\end{proposition}

The first conclusion identifies an incentive failure: without hard-positive signal, even an optimal generator can remain confined to easy evidence, \emph{over-}abstaining from answering slightly harder but still answerable inputs. The second shows that sufficient augmentation weight removes all optima in this conservative regime. Augmentation thus counteracts a failure of DAC's cross-verification by making acceptance of harder sufficient evidence rewarding. This mechanism is consistent with the reduced non-abstention recall and lower answer accuracy observed when augmentation is removed. The appendix relaxes the restriction on ordinary sufficient evidence, gives an explicit bound on $\Gamma$, and further relates the analysis to an idealized Stackelberg game \citep{stackelberg}.

\section{Related work}
\label{sec:related_work}
\paragraph{Learning to search with RL}
Retrieval-augmented reasoning is increasingly formulated as an agentic task, building on ReAct-style interleaving of reasoning and tool use \citep{react}. Early RL approaches optimize such trajectories with outcome-level rewards \citep{searchr1,r1searcher,research}, while later work develops finer turn-level credit assignment from these signals \citep{igpo,tips,oases,trace}. This paradigm has been extended to complex, long-horizon deep-research settings, where models perform many rounds of search, browsing, and synthesis \citep{sfrdeepresearch,tongyidr}. A recent direction is to decompose the process across specialized agents \citep{s3,dart,mask}. These methods vary in role decomposition and how each agent is trained, but largely rely on final outcome-level supervision. DAC follows this multi-agent direction while introducing role-specific credit assignment between search and generation.

\paragraph{Multi-agent LLM training}
Recent work has explored multi-agent LLM training through cooperative and adversarial interactions. Cooperative approaches train role-specialized agents to jointly solve tasks through reasoning, refinement, or discussion \citep{malt,maporl}. Related work on verifier-guided RL jointly trains verifiers or critics with the policy, allowing the supervision signal to adapt as the policy evolves \citep{seed15thinking,kirchner2024pvg,rlac}; some further introduce adversarial interactions to expose failure modes and improve robustness. DAC is related to these ideas in that the generator acts as a critic of search, but it also differs in that both agents are task-performing components: the generator produces the final response while simultaneously providing a co-evolving verification signal that shapes the searcher's behavior.

\section{Limitations and Future Work}
\label{sec:conclusion}
While DAC performs consistently across our evaluated settings, several limitations remain and motivate future work. First, our hard-positive evidence augmentation is a simplified approximation that adds irrelevant or weakly related documents to sufficient contexts. Real settings may involve conflicting, misleading, partially sufficient, or only jointly sufficient information. Future work could construct more diverse hard-positives through structured perturbations or learned and adversarial evidence generation, potentially adapting the difficulty of evidence to the evolving generator.

Also, for efficiency, we instantiate search sufficiency verification using substring EM. This lightweight proxy can miss paraphrased or jointly sufficient evidence and may falsely accept incidental answer-string matches. DAC is agnostic to the implementation of the verifier, but it is not independent of verifier quality. Future work should evaluate semantic and task-specific verifiers, as well as robustness to verifier false positives and false negatives.

Finally, due to computational constraints, our dedicated-policy experiments are limited to LoRA adaptation. Whether the benefits of role specialization persist under independent full-parameter fine-tuning remains open. Future work could investigate full-parameter training of dedicated policies and characterize its trade-offs in specialization, training stability, and computational cost.

\subsection*{AI use statement}
In this work, we used generative AI tools to assist in implementing methods, specifically in writing and debugging parts of the training and evaluation code. We have not used generative AI tools to propose or refine hypotheses, design experiments, generate synthetic datasets, clean or reformat datasets, or interpret results. Translation and qualitative data analysis are not applicable to this work. Additionally, we used generative AI tools to search for and identify related work and to edit the manuscript for grammar and readability. We have reviewed all AI-assisted work; AI-assisted code was inspected and tested by the authors, all reported results come from our own training and evaluation runs, every cited work was checked by the authors, and all bibliographic entries were verified against the original sources. We take responsibility for the final content of this work, including text, claims, or artifacts produced with the aid of generative AI.

\subsubsection*{Acknowledgments}
This work was supported in part by the National Research Foundation of Korea (NRF) grant (RS-2023-00280883, RS-2023-00222663); by the National Research Foundation, Korea, under project BK21 FOUR (Dept. of Data Science, SNU, No. 5199990914569); by the Korea Institute of Science and Technology Information (KISTI) in 2026 (No. (KISTI)K26L3M1C1), aimed at developing KONI (KISTI Open Neural Intelligence), a large language model specialized in science and technology; and by the Institute of Information \& communications Technology Planning \& Evaluation (IITP) grant funded by the Korea government (MSIT) (RS-2025-02263754, Human-Centric Embodied AI Agents with Autonomous Decision-Making; RS-2025-25442149, LG AI STAR Talent Development Program for Leading Large-Scale Generative AI Models in the Physical AI Domain; and RS-2025-25463302, Global Talent Recruitment for AI-powered Drug Discovery); by grant (25202MFDS003) from Ministry of Food and Drug Safety in 2025; Institute of Information \& communications Technology Planning \& Evaluation (IITP) grant funded by the Korea government (MSIT).

\bibliographystyle{plainnat_with_preprint}
\bibliography{iclr2027_conference}

\newpage
\appendix
\section{Theoretical Analysis of DAC}
\label{app:theoretical_analysis}

\subsection{Setup and generator objective}
\label{app:dac_augmentation_setup}

We analyze the generator reward
\begin{equation*}
    r_\mathcal{G} = sa + (1-s)v,
\end{equation*}
where $s$ denotes search sufficiency, $v$ abstention, and $a$ answer correctness. Thus, on sufficient evidence ($s=1$), the generator is rewarded only for answering correctly, whereas on insufficient evidence ($s=0$), it is rewarded only for abstaining. We assume accurate sufficiency and correctness labels.

Each input $(Q, \mathcal{E})$ has a scalar \textbf{difficulty score} $h \in [0,1]$, with larger scores representing inputs the generator is more inclined to abstain. The generator chooses an \textbf{acceptance threshold} $\theta \in \qty[0, 1]$, and uses
\begin{equation*}
    v \gets v_\theta(h) \coloneqq \mathbb{I}\bigl[h > \theta\bigr],
\end{equation*}
i.e., it answers on an evidence set of difficulty $h$ if and only if $h \leq \theta$.

Sufficient and insufficient inputs may have overlapping score distributions; their true sufficiency labels are not supplied to the acceptance rule. Conditional on the abstention decision, we assume the answer generation policy (when chosen to answer) is held fixed as $\theta$ changes. This is an idealized restriction used to isolate the acceptance decision.

In a batch training setting, the generator is updated on a mixture of evidence states. We abstract this mixture using three effective generator-side weights: $\alpha$ for \textbf{insufficient}-evidence examples, $\beta$ for \textbf{ordinary sufficient} examples, and $\gamma$ for \textbf{augmented sufficient} examples. Insufficient examples may arise from failed search trajectories or constructed insufficient inputs. We aggregate these cases into the generator-side weight $\alpha$, rather than modeling search failure explicitly, since our analysis focuses on the impact of the acceptance threshold.

Let $I$, $H$, and $S$ denote the three sources of inputs in order. For each sufficient source $j \in \qty{H,S}$, let the measurable function $c_j(h)\in[0,1]$ give the conditional probability of a correct answer when chosen to answer. No monotonicity assumption on the correctness functions is required. Also, let $P_j$ be the probability distribution over difficulty scores for each evidence source. We assume that $P_I$ has no atoms, while $P_H$ and $P_S$ may contain point masses. The former assumption holds for continuous score distributions and ensures continuity of the insufficient-input acceptance mass. The latter allowance includes the endpoint construction in Appendix~\ref{app:dac_stackelberg_special_case}. We hold the evidence distributions, correctness functions, and the original source weights $\alpha$ and $\beta$ fixed while varying only the augmentation weight $\gamma$. 

Define the cumulative reward quantities
\begin{align*}
    F_I(\theta) &= P_I\bigl([0,\theta]\bigr), \\
    C_j(\theta)&=\int_{[0,\theta]}c_j(h)\,dP_j(h),
    \qquad j \in \qty{H,S}.
\end{align*}
All three are nondecreasing. $F_I$ is the probability mass of insufficient inputs incorrectly accepted at threshold $\theta$, while $C_H$ and $C_S$ are correctness rewards gained by answering sufficient inputs from the two sources.

For an insufficient input at score $h$, answering receives reward $0$ and abstaining receives reward $1$. Its expected reward is therefore $v_\theta(h)$. For a sufficient input from source $j \in \qty{H,S}$, the expected reward is $\bigl(1-v_\theta(h)\bigr)c_j(h)$. Integrating over the three sources gives the expected generator reward
\begin{equation*}
    \mathbb{E}\qty[r_\mathcal{G} \bigm\vert \theta] = \frac{\alpha\bigl(1-F_I(\theta)\bigr) + \beta C_H(\theta) + \gamma C_S(\theta)}{\alpha + \beta + \gamma}.
\end{equation*}
The denominator is positive and independent of $\theta$. Hence the following unnormalized objective has the same maximizers for each fixed augmentation weight $\gamma$:
\begin{equation} \label{eq:dac_augmentation_objective}
    U_\gamma(\theta) = \alpha\bigl(1-F_I(\theta)\bigr) + \beta C_H(\theta) + \gamma C_S(\theta).
\end{equation}

\subsection{Formal statement and proof of \texorpdfstring{\Cref{prop:dac_regime_change}}{Proposition~\ref{prop:dac_regime_change}}}
\label{app:dac_regime_change_proof}

\begin{theorem}[Formal restatement of \Cref{prop:dac_regime_change}] \label{thm:dac_regime_change}
    Under the assumptions in Appendix~\ref{app:dac_augmentation_setup}, hold $P_I$, $P_H$, $P_S$, $c_H$, $c_S$, $\alpha > 0$, and $\beta \geq 0$ fixed. For each $\gamma \geq 0$, let
    \begin{equation*}
        \Theta(\gamma) = \argmax_{\theta \in [0, 1]}U_\gamma(\theta).
    \end{equation*}
    Suppose there exists $b \in [0,1)$ such that
    \begin{equation} \label{eq:dac_main_formal_conditions}
        P_H\bigl((b,1]\bigr) = 0, \qquad F_I(\theta) > F_I(b) \quad \text{for every} \;\; \theta > b, \qquad C_S(1) > C_S(b).
    \end{equation}
    Write $U_\gamma(\theta)= U_0(\theta) + \gamma C_S(\theta)$ and define
    \begin{equation*}
        M_b = \max_{0 \leq u \leq b}U_0(u), \qquad \Gamma_b = \frac{M_b - U_0(1)}{C_S(1) - C_S(b)}.
    \end{equation*}
    Then, the following hold.
    \begin{enumerate}[leftmargin=*,label=(\roman*),itemsep=0pt,topsep=0pt]
        \item $\Theta(\gamma)$ is nonempty and compact for every $\gamma \geq 0$.
        \item $0 < \Gamma_b < \infty$.
        \item $\Theta(0) \subseteq [0,b]$, while $\Theta(\gamma) \subseteq (b,1]$ for every $\gamma > \Gamma_b$.
        \item For every $\gamma > \Gamma_b$ and every $\theta^\star \in \Theta(\gamma)$,
        \begin{equation} \label{eq:dac_additional_accepted_correctness}
            C_S(\theta^\star) - C_S(b) \geq \frac{U_0(b) - U_0(\theta^\star)}{\gamma} > 0.
        \end{equation}
        In particular, $P_S\bigl((b,\theta^\star]\bigr)>0$, so non-abstention recall on the fixed augmented source strictly exceeds that of every unaugmented optimum.
    \end{enumerate}
\end{theorem}

The first condition in \cref{eq:dac_main_formal_conditions} restricts ordinary sufficient inputs to difficulty at most $b$. The second means that any broader threshold admits additional insufficient inputs and therefore loses reward. The third requires that there exist augmented sufficient evidence above $b$ on which the generator can answer correctly with positive probability. These conditions describe a shortage of hard-positive signal in the original mixture that can be addressed through augmentation. \Cref{prop:dac_regime_change} in the main text follows by taking $\Gamma = \Gamma_b$.

\begin{proof}
First, optimal thresholds exist. Because $P_I$ is atomless, its cumulative function $F_I$ is continuous. Also, $C_H$ and $C_S$ are both cumulative functions of nonnegative finite measures and are upper semicontinuous. Thus $U_\gamma$ is upper semicontinuous, and its maximizing set on $[0,1]$ is nonempty and compact. Likewise, $U_0$ attains its maximum on $[0,b]$, so $M_b$ is well defined.

Next, the numerator of $\Gamma_b$ is positive because
\begin{equation*}
    M_b \geq U_0(b) > U_0(1),
\end{equation*}
where the strict inequality follows from $P_H\bigl((b,1]\bigr)=0$ and $F_I(1)>F_I(b)$. Its denominator is positive by assumption. All terms are finite, proving $0<\Gamma_b<\infty$.

The first condition of \cref{eq:dac_main_formal_conditions} gives $C_H(\theta) = C_H(b)$ for every $\theta > b$. Consequently, without augmentation, or if $\gamma = 0$,
\begin{equation} \label{eq:dac_conservative_regime}
    U_\gamma(\theta) - U_\gamma(b) = U_0(\theta) - U_0(b) = -\alpha\bigl(F_I(\theta) - F_I(b)\bigr) < 0, \qquad \theta > b.
\end{equation}
Hence no threshold larger than $b$ can maximize $U_\gamma$, proving $\Theta(0) \subseteq [0, b]$. Note that since $C_S(1) > C_S(b)$, these thresholds in $\Theta(0)$ reject inputs with positive correctness-weighted mass under $P_S$. The rejected mass is at least $C_S(1) - C_S(b) > 0$.

Also, for any $u \leq b$ and $\gamma > \Gamma_b$,
\begin{align*}
    U_\gamma(1) - U_\gamma(u)
    & = U_0(1)-U_0(u) + \gamma\bigl(C_S(1) - C_S(u)\bigr) \\
    & \geq U_0(1)-M_b + \gamma\bigl(C_S(1) - C_S(b)\bigr) > 0.
\end{align*}
Hence $\Theta(\gamma)\subseteq(b,1]$.

Finally, choose any $\theta^\star\in\Theta(\gamma)$, then optimality gives $U_\gamma(\theta^\star)\geq U_\gamma(b)$. If $\gamma > \Gamma_b$, then $\theta^\star > b$ and $U_0(\theta^\star) < U_0(b)$ by \cref{eq:dac_conservative_regime}, so
\begin{equation*}
    \gamma\bigl(C_S(\theta^\star) - C_S(b)\bigr) \geq U_0(b) - U_0(\theta^\star) > 0.
\end{equation*}
Dividing by $\gamma>0$ proves \cref{eq:dac_additional_accepted_correctness}. Thus $C_S(\theta^\star)>C_S(b)$, i.e., the optimum accepts additional sufficient inputs with positive correctness reward. In particular, $P_S\bigl((b,\theta^\star]\bigr) > 0$. Every unaugmented optimum is at most $b$, so its acceptance probability under $P_S$ is at most $P_S\bigl([0,b]\bigr)$, which is strictly smaller than $P_S\bigl([0,\theta^\star]\bigr)$. This explains the increase in non-abstention recall when using augmented samples.
\end{proof}

\paragraph{Beyond the support restriction}
Previously, we assumed that there is no ordinary sufficient evidence with difficulty above $b$. This is a sufficient condition for the conservative regime to exist, but it is not necessary. The same conclusions and proof hold under the more general condition
\begin{equation*}
    \beta\bigl(C_H(\theta) - C_H(b)\bigr) < \alpha\bigl(F_I(\theta) - F_I(b)\bigr), \qquad \text{for every} \;\; \theta > b.
\end{equation*}
This allows ordinary hard-positives, provided their correctness reward is too weak to offset the reward lost from accepting additional insufficient inputs. In either case, the optimal acceptance threshold is at most $b$ without augmentation, and strictly larger than $b$ with sufficient augmentation weight. The bound $\Gamma_b$ is sufficient but not necessary for the regime change; a smaller weight may suffice if the ordinary sufficient evidence above $b$ is weak enough.

\paragraph{Implication for cross-verification}
On sufficient evidence, the searcher reward equals the generator's acceptance indicator, $1-v$. A threshold at or below $b$ therefore gives the searcher zero reward on sufficient inputs above $b$, even when their conditional answering accuracy is positive. The augmented regime restores that reward on additional answerable sufficient inputs. Yet this is an incentive-level consequence of the reward definition, and it does not by itself prove that searcher training reaches those inputs or that the coupled training process converges.

\subsection{Monotonicity of the optimal acceptance threshold}
\label{app:dac_monotonicity_proof}

We further analyze the optimal acceptance threshold as a function of the augmentation weight $\gamma$. Define
\begin{equation*}
    \underline{\theta}(\gamma) = \min\Theta(\gamma), \qquad \overline{\theta}(\gamma) = \max\Theta(\gamma).
\end{equation*}
The following result shows that both extremal selections of optimal thresholds are nondecreasing in $\gamma$. This extends \Cref{thm:dac_regime_change} and does not require the support restriction on ordinary sufficient evidence.

\begin{proposition}[Monotone acceptance boundary] \label{prop:dac_monotone_boundary}
    Under the assumptions in Appendix~\ref{app:dac_augmentation_setup}, hold the evidence distributions, conditional correctness functions, $\alpha$, and $\beta$ fixed. Then, both $\underline{\theta}(\gamma)$ and $\overline{\theta}(\gamma)$ are nondecreasing in $\gamma$. Under either consistent choice, non-abstention recall is nondecreasing on any fixed distribution of difficulty scores for sufficient inputs.
\end{proposition}

\begin{proof}
Existence and compactness of the maximizing set were established in \Cref{thm:dac_regime_change}, so both extrema are well defined.

Fix $\gamma_2>\gamma_1$ and choose $\theta_i\in\Theta(\gamma_i)$. Their optimality implies
\begin{align*}
    U_0(\theta_1) + \gamma_1C_S(\theta_1) & \geq U_0(\theta_2) + \gamma_1C_S(\theta_2), \\
    U_0(\theta_2) + \gamma_2C_S(\theta_2) & \geq U_0(\theta_1) + \gamma_2C_S(\theta_1).
\end{align*}
Adding these inequalities yields
\begin{equation} \label{eq:dac_monotonicity_key}
    (\gamma_2-\gamma_1)\bigl(C_S(\theta_2)-C_S(\theta_1)\bigr) \geq 0.
\end{equation}
Suppose $\theta_2 < \theta_1$. Because $C_S$ is nondecreasing, \cref{eq:dac_monotonicity_key} forces $C_S(\theta_2) = C_S(\theta_1)$. The optimality inequalities then give $U_0(\theta_2) = U_0(\theta_1)$. Both thresholds therefore maximize both objectives. If $\theta_i = \underline{\theta}(\gamma_i)$, this contradicts the minimality of $\theta_1$, because the smaller $\theta_2$ also maximizes $U_{\gamma_1}$. If $\theta_i = \overline{\theta}(\gamma_i)$, it contradicts the maximality of $\theta_2$, because the larger $\theta_1$ also maximizes $U_{\gamma_2}$. Both extremal selections are consequently nondecreasing.

For any fixed sufficient-input difficulty score distribution $P^+$, non-abstention recall equals $P^+\bigl([0,\theta]\bigr)$. This is nondecreasing in $\theta$, proving the recall statement.
\end{proof}

Therefore, we can conclude that with larger augmentation weight, the generator never reduces its acceptance threshold. Furthermore, if $C_S$ is strictly increasing, \cref{eq:dac_monotonicity_key} also rules out $\theta_2 < \theta_1$ for arbitrary choices of maximizers, implying that the entire set of maximizers is nondecreasing in $\gamma$.

\subsection{Conditions for strict expansion}
\label{app:dac_expansion_conditions}

We now quantify conditions under which increasing the augmentation weight \emph{strictly} increases the optimal acceptance threshold. Suppose the three score distributions have densities $f_I,f_H,f_S$ and the products $c_Hf_H$ and $c_Sf_S$ are continuous. Then the derivative of the objective is
\begin{equation} \label{eq:dac_marginal_reward}
    U_\gamma'(\theta) = -\alpha f_I(\theta) + \beta c_H(\theta)f_H(\theta) + \gamma c_S(\theta)f_S(\theta).
\end{equation}
Any interior optimum $\theta^\star$ therefore satisfies
\begin{equation*}
    \beta c_H(\theta^\star)f_H(\theta^\star) + \gamma c_S(\theta^\star)f_S(\theta^\star) = \alpha f_I(\theta^\star).
\end{equation*}
The left-hand side is the correctness reward gained by admitting marginal sufficient inputs, and the right-hand side is the reward lost by admitting marginal insufficient inputs. This is a necessary condition. If $U_\gamma'$ is strictly decreasing and changes sign from positive to negative, its unique root is the unique global optimum.

\begin{proposition}[Strict local expansion] \label{prop:dac_strict_expansion}
Suppose the terms in \cref{eq:dac_marginal_reward} are continuously differentiable. Fix $\gamma_0>0$ and assume the global optimum
\begin{equation*}
    \theta^\star(\gamma_0) = \argmax_{\theta \in [0, 1]}U_{\gamma_0}(\theta)
\end{equation*}
is unique and interior throughout a neighborhood of $\gamma_0$, with $U_{\gamma_0}''\bigl(\theta^\star(\gamma_0)\bigr) < 0$. Then
\begin{equation} \label{eq:dac_boundary_sensitivity}
    \left.\frac{d\theta^\star}{d\gamma}\right|_{\gamma=\gamma_0} = \frac{c_S\bigl(\theta^\star(\gamma_0)\bigr)f_S\bigl(\theta^\star(\gamma_0)\bigr)}{-U_{\gamma_0}''\bigl(\theta^\star(\gamma_0)\bigr)} \geq 0.
\end{equation}
The derivative is strictly positive if its numerator is positive.
\end{proposition}

\begin{proof}
An interior maximizer $\theta^\star$ satisfies $U_\gamma'(\theta^\star) = 0$. Continuity and compactness imply that unique maximizers converge to $\theta^\star(\gamma_0)$ as $\gamma\to\gamma_0$, and the root is locally unique. Also, since $U_{\gamma_0}''\bigl(\theta^\star(\gamma_0)\bigr)<0$, we can apply the implicit function theorem to the stationarity equation $U_\gamma'(\theta^\star) = 0$ in a neighborhood of $\gamma_0$. Differentiating the stationarity equation with respect to $\gamma$ gives
\begin{equation*}
    U_\gamma''(\theta^\star)\frac{d\theta^\star}{d\gamma} + c_S(\theta^\star)f_S(\theta^\star) = 0,
\end{equation*}
which yields \cref{eq:dac_boundary_sensitivity}. The denominator is strictly positive and the numerator is nonnegative, proving the nonnegativity of the derivative. If $c_S(\theta^\star)f_S(\theta^\star) > 0$, then the derivative is strictly positive.
\end{proof}

Therefore, we can state that the optimal threshold strictly increases until its movement offsets the extra marginal reward obtained from the augmented source. In other words, the threshold continues to increase until either the marginal correctness reward from the augmented source becomes zero or the optimum reaches the boundary of the feasible interval. The former may occur if augmented sufficient evidence has no mass near the current optimum, while the latter occurs when the optimum reaches $1$.

\subsection{An idealized Stackelberg special case}
\label{app:dac_stackelberg_special_case}

We can simplify the analysis by considering a special case in which the evidence distributions are concentrated at a point mass. This allows us to view the system as a Stackelberg game, especially a lightweight analogue of a \emph{prover-verifier game} \citep{kirchner2024pvg,anil2021pvg}. \citet{kirchner2024pvg} study a game where the main agent, the \emph{prover}, is trained to produce correct answers, while a secondary agent, the \emph{verifier}, is trained to accept or reject the prover's answers. To guide the verifier to learn a robust decision boundary, the prover is trained in two tracks: a \emph{helpful prover} that produces correct answers that should be accepted and a \emph{sneaky prover} that produces incorrect answers that may fool the verifier. This setup encourages the verifier to learn a decision boundary that can distinguish between correct and incorrect answers, even in the presence of adversarial examples.

Our augmentation targets the opposite failure mode: \emph{rejecting correct but noisy answerable cases}. Thus, our sneaky agent should target examples that are correct and answerable but likely to be rejected because the evidence is noisy or imperfect. Instead of training a separate adversarial agent to produce such hard-positives, we approximate this effect with noise-based evidence augmentation. In more complex settings where noise sampling is difficult, one could instead train an adversarial searcher to generate hard sufficient evidence.

Under this interpretation, we have a game of three players: a \textbf{helpful searcher} $\mathcal{S}_H$, a \textbf{sneaky searcher} $\mathcal{S}_S$, and a \textbf{generator} $\mathcal{G}$. The helpful searcher aims to produce sufficient evidence that the generator accepts and answers from, whereas the sneaky searcher aims to produce sufficient but difficult evidence that the generator incorrectly rejects by abstaining. The generator becomes the leading agent and commits to an acceptance threshold $\theta \in [0, 1]$. The two searchers then choose their evidence difficulties $h_H, h_S \in [0, 1]$ in response to the generator's commitment. Under the source weights $\alpha, \beta, \gamma$, we can write the utility functions of the three players as
\begin{align}
    U_H(h_H,\theta) &= \mathbb{I}\bigl[h_H \leq \theta\bigr], \notag\\
    U_S(h_S,\theta) &= \mathbb{I}\bigl[h_S > \theta\bigr], \notag\\
    U_G(h_H,h_S,\theta) &= \alpha(1 - \theta) + \beta c_G(h_H)\mathbb{I}\bigl[h_H \leq \theta\bigr] + \gamma c_G(h_S)\mathbb{I}\bigl[h_S \leq \theta\bigr],
\end{align}
where $\alpha>0$, $\beta,\gamma \geq 0$. Here, $c_G : [0, 1] \to [0,1]$ is the probability that the generator answers correctly when given a sufficient evidence set of difficulty $h$. We assume $c_G(0) = 1$, or equivalently, that the generator always answers correctly on the easiest sufficient evidence. Insufficient evidence contributes the first term, and it is not chosen by the two searchers in this simplified game. Among best responses, let the helpful searcher choose the smallest difficulty and the adversarial searcher the largest.

A generator-leader Stackelberg equilibrium is a triple $(h_H^\star, h_S^\star, \theta^\star)$ in which the searchers follow these best-response rules and $\theta^\star$ maximizes the resulting generator utility. We can characterize the equilibrium triples in closed form.
\begin{proposition}[Endpoint equilibria] \label{prop:dac_endpoint_equilibria}
    Under the above assumptions, the set of Stackelberg equilibria of the induced game is given by:
    \begin{equation} \label{eq:dac_endpoint_equilibria}
        \qty(h_H^\star, h_S^\star, \theta^\star) = \begin{cases}
            (0, 1, 0), & \text{if} \;\; \alpha > \gamma c_\mathcal{G}(1), \\
            (0, 1, 1), & \text{if} \;\; \gamma c_\mathcal{G}(1) > \alpha, \\
            (0, 1, 0) \text{ or } (0, 1, 1), & \text{if} \;\; \alpha = \gamma c_\mathcal{G}(1).
        \end{cases}
    \end{equation}
\end{proposition}

\begin{proof}
For any $\theta$, the helpful searcher's best-response set is $[0, \theta]$, so its tie-breaking rule selects $h_H = 0$. For $\theta < 1$, the sneaky searcher's best-response set is $(\theta, 1]$, so it selects $h_S = 1$. At $\theta = 1$, every difficulty gives it utility zero, and the same rule again selects $h_S = 1$. The induced generator objective then becomes
\begin{equation*}
    U_G(0,1,\theta) = \beta + \alpha(1 - \theta) + \gamma c_G(1)\mathbb{I}\bigl[\theta = 1\bigr].
\end{equation*}
It is strictly decreasing on $[0,1)$, so only $0$ and $1$ can be optimal. Comparing their values $\beta + \alpha$ and $\beta + \gamma c_G(1)$ proves \cref{eq:dac_endpoint_equilibria}, including the equality case.
\end{proof}

This induced objective is exactly \cref{eq:dac_augmentation_objective} with $P_I$ uniform, $P_H = \delta_0$, $P_S = \delta_1$, $c_H(0) = 1$, and $c_S(1) = c_G(1)$. This special case illustrates how the sneaky searcher plays a central role in determining the generator's equilibrium acceptance threshold. At equilibrium, the helpful searcher always chooses the easiest sufficient evidence, $h_H^\star = 0$, while the sneaky searcher chooses the hardest sufficient evidence, $h_S^\star = 1$. The generator's threshold is determined by a trade-off between the reward for abstaining on insufficient evidence, weighted by $\alpha$, and the reward for answering on hard but sufficient evidence, weighted by $\gamma c_{\mathcal{G}}(1)$. That is, when the generator is sufficiently accurate on hard-evidence inputs and the effective proportion of hard-sufficient examples is large relative to insufficient examples, the equilibrium shifts toward a larger acceptance threshold. In this regime, we expect the generator to abstain less often on hard but sufficient evidence. Also, without the sneaky searcher, i.e., when $\gamma = 0$, the equilibrium deterministically collapses to the trivial easy-positive regime $(h_H^\star,\theta^\star) = (0,0)$. In this regime, the generator has no incentive to expand its acceptance threshold beyond the easiest sufficient evidence.

\paragraph{Scope and limitations of our analysis}
Our general analysis establishes no convergence guarantee for actual training. It also does not imply higher forced-answer accuracy or selective accuracy, and we assume the generator's answering ability is fixed. Moreover, expanding acceptance can admit additional insufficient inputs, but we do not model any adverse effect on the searcher.

\section{Further implementation details}
\label{app:further_implementation_details}

\paragraph{Retrieval setup}
The dense retrieval system uses the December 2018 Wikipedia dump as its corpus, with \texttt{E5-base-v2} \citep{e5} as the embedding model. Each query retrieves the top-$3$ passages, and the searcher is allowed at most four search turns.

\paragraph{Models and parameterization}
We train DAC with \texttt{Qwen2.5-7B-Instruct} and \texttt{Qwen3-8B}. While \texttt{Qwen3-8B} supports both thinking and non-thinking modes, we use the non-thinking mode to avoid starting from a model with previously shaped reasoning behavior. The main experiments use full-parameter fine-tuning of a single shared policy. For the policy-sharing analysis in \cref{sec:shared_dedicated}, we instead use LoRA. The shared-policy setting uses a rank-$64$ adapter with $\alpha=128$, while the dedicated-policy setting uses separate rank-$32$ adapters with $\alpha=64$ for the searcher and generator.

\paragraph{Training setup}
RL training is implemented with verl \citep{verl}. We use FSDP2 \citep{fsdp} with CPU offloading and gradient checkpointing, and generate rollouts with vLLM \citep{vllm} using tensor parallel size $1$. Rollouts use temperature $1.0$ and top-$p$ $1.0$. We train on $8$ H100 GPUs with a global batch size of $512$, mini-batch size of $256$, and dynamic micro-batch sizing. We use AdamW \citep{adamw} with $\beta_1=0.9$, $\beta_2=0.999$, and weight decay $0.01$. The main experiments use a learning rate of $1\times10^{-6}$. For shared-policy LoRA training, we apply a $10\times$ learning-rate scaling following \citet{schulman2025lora}. For dedicated-policy LoRA training, both agents use a learning rate of $1\times10^{-6}$ on \texttt{Qwen2.5-7B-Instruct}. On \texttt{Qwen3-8B}, we use $2\times10^{-6}$ for the searcher and $1\times10^{-5}$ for the generator based on stabilization tuning. We train for $500$ steps and save checkpoints every $20$ steps. If training diverges, we report the last stable checkpoint based on the generator reward curve.

\paragraph{RL objectives}
Because the searcher receives intermediate rewards across multiple search turns, we optimize it with PPO \citep{ppo}. In contrast, generator training can be viewed as a single-turn LLM RL problem: given $(Q,\mathcal{E})$, the generator produces reasoning followed by either an answer or an abstention decision. We therefore optimize the generator with GRPO \citep{deepseekmath}. For both objectives, we use token-level policy-gradient losses \citep{dapo}:
\begin{align*}
    \mathcal{J}_{\text{PPO}}(\theta) &= \mathbb{E}_{\begin{subarray}{l}x \sim \mathcal{D} \\ y \sim \pi_0(\cdot \mid x)\end{subarray}}\qty[\frac{1}{\abs{y}}\sum_{t=1}^{\abs{y}}\mathcal{C}_\epsilon\qty(\frac{\pi_\theta(y_t \mid x, y_{<t})}{\pi_0(y_t \mid x, y_{<t})}, A^{\text{GAE}}_t)], \\
    \mathcal{J}_{\text{GRPO}}(\theta) &= \mathbb{E}_{\begin{subarray}{l}x \sim \mathcal{D} \\ \qty{y_i}_{i=1}^{G} \sim \pi_0(\cdot \mid x)\end{subarray}}\qty[\frac{1}{\sum_{i=1}^{G}\abs{y_i}}\sum_{i=1}^{G}\sum_{t=1}^{\abs{y_i}}\mathcal{C}_\epsilon\qty(\frac{\pi_\theta(y_{i,t} \mid x, y_{i,<t})}{\pi_0(y_{i,t} \mid x, y_{i,<t})}, A^{\text{GRPO}}_i)],
\end{align*}
where $\mathcal{C}_\epsilon$ is the clipping function $\mathcal{C}_\epsilon\qty(\rho, A) = \min\bigl(\rho A, \operatorname{clip}\qty(\rho, 1-\epsilon, 1+\epsilon)A\bigr)$. We use no KL penalty and set the clipping parameter to $\epsilon=0.2$. For PPO, we set $\gamma=1$ and $\lambda=0.995$. For generator training with GRPO, we use group size $G=5$. Here, $\mathcal{D}$ denotes the prompt dataset, $\pi_\theta$ the current policy, $\pi_0$ the rollout policy, and $A_t^{\text{GAE}}$ and $A_i^{\text{GRPO}}$ the corresponding GAE and GRPO advantages. For mini-batch and micro-batch aggregation, losses are normalized by the total number of tokens within each per-GPU micro-batch and then averaged across micro-batches.

\paragraph{Hard-positive evidence augmentation}
For noise-based augmentation, we retrieve the top-$15$ passages using the original question, remove passages already present in the evidence set, and select three remaining passages in reverse rank order as distractors. These are appended to the original evidence to form an augmented input used for generator training. Augmented samples only increase the number of samples contributing to the same update, and do not introduce extra optimization steps.

\section{Baselines}
\label{app:baselines}

We provide brief descriptions and reproduction details for each baseline below. Unless otherwise noted, all methods use the same backbone models, retrieval corpus, and evaluation setup as DAC.

\begin{itemize}[leftmargin=1em]
\item \textbf{Direct prompting} asks the backbone model to answer the question directly without external retrieval.

\item \textbf{CoT} \citep{cot} additionally elicits chain-of-thought reasoning by appending \texttt{Let's think step by step.} to the prompt.

\item \textbf{Naive RAG} \citep{rag} supplies the model with the top-$3$ passages retrieved by our dense-retrieval system and asks it to generate the final answer.

\item \textbf{IRCoT} \citep{ircot} interleaves chain-of-thought reasoning with retrieval, using intermediate reasoning to guide subsequent searches. We use the official codebase.\footnote{\url{https://github.com/StonyBrookNLP/ircot}} For the \texttt{Qwen2.5-7B-Instruct} setting, we instead use \texttt{Qwen2.5-7B-Base}, since the instruction-tuned model performs poorly under the generation format assumed by IRCoT.

\item \textbf{Search-o1} \citep{searcho1} dynamically invokes search during reasoning and uses a separate Reason-in-Documents component to process retrieved information. We reproduce it using the official codebase.\footnote{\url{https://github.com/RUC-NLPIR/Search-o1}}

\item \textbf{Search-R1} \citep{searchr1,searchr1_2} trains a single model with RL to interleave reasoning and multi-turn search. We consider four variants combining PPO or GRPO with either outcome-only rewards or outcome and format rewards.
All results and analysis (\Cref{tab:main_results}; \Cref{fig:disentangled_results}; \Cref{fig:abstention}; \Cref{tab:llm_judge_results}) are reported against the best performing variant of each model family: GRPO with outcome-only reward for \texttt{Qwen2.5-7B-Instruct}, GRPO with outcome and format rewards for \texttt{Qwen3-8B}. We further report entire results for all variants in \cref{app:search_r1_full}.
All variants are reproduced using the official codebase,\footnote{\url{https://github.com/PeterGriffinJin/Search-R1}} except for three \texttt{Qwen2.5-7B-Instruct} configurations for which official checkpoints are available: PPO with outcome-only reward, GRPO with outcome-only reward, and GRPO with outcome and format rewards.\footnote{%
\begin{tabular}[t]{@{}l@{}}
    \href{https://huggingface.co/PeterJinGo/SearchR1-nq_hotpotqa_train-qwen2.5-7b-it-em-ppo-v0.2}{\texttt{PeterJinGo/SearchR1-nq\_hotpotqa\_train-qwen2.5-7b-it-em-ppo-v0.2}} \\
    \href{https://huggingface.co/PeterJinGo/SearchR1-nq_hotpotqa_train-qwen2.5-7b-it-em-grpo-v0.2}{\texttt{PeterJinGo/SearchR1-nq\_hotpotqa\_train-qwen2.5-7b-it-em-grpo-v0.2}} \\
    \href{https://huggingface.co/PeterJinGo/SearchR1-nq_hotpotqa_train-qwen2.5-7b-it-em-grpo-v0.3}{\texttt{PeterJinGo/SearchR1-nq\_hotpotqa\_train-qwen2.5-7b-it-em-grpo-v0.3}}
\end{tabular}%
}
For the format-reward variants, we use a coefficient of $0.2$. We train for up to $500$ steps, saving checkpoints every $20$ steps and using the last stable checkpoint if training collapses.

\item \textbf{TIPS} \citep{tips} augments single-agent search training with turn-level reward shaping based on the information contribution of each turn. We use the official codebase.\footnote{\url{https://github.com/ucsd-wang-lab-lm/tips}} Following the paper, we train for $500$ steps and re-synchronize the teacher at step $200$. We use dynamic $\alpha$ with a target information-reward range of $0.05$--$0.3$, which is reported as the best-performing configuration.

\item \textbf{S3} \citep{s3} separates search from answer generation, keeps the generator fixed, and trains the searcher based on downstream answer quality. We use the official codebase.\footnote{\url{https://github.com/pat-jj/s3}} S3 introduces its own answer correctness verifier (\texttt{GenAcc}), but for consistency with the other methods, we instead use EM as the verifier during training. We train for $100$ steps, exceeding the paper's observation that roughly $15$ steps are sufficient for convergence. In our reproduction, however, neither training nor test performance improved meaningfully over the course of training. We therefore report the final trained checkpoint for consistency with the other baselines.

\item \textbf{DART} \citep{dart} disentangles reasoning and tool-use behavior using separate LoRA modules and optimizes them toward the final task outcome. We use the associated codebase provided by one of the authors.\footnote{\url{https://github.com/sheriyuo/DART}} The originally reported DART results additionally use an SFT warm-up stage and add special tags such as \texttt{<search>} to the tokenizer. To keep the training setup comparable to the other baselines, we omit these steps and train DART directly from the original backbone. Switching between the two roles is instead triggered whenever the generated string \texttt{<search>} appears, even though it is not represented as a single token. We train for up to $500$ steps, saving checkpoints every $20$ steps and using the last stable checkpoint if training collapses.

\item \textbf{M-ASK} \citep{mask} decomposes agentic search into specialized components for search behavior and knowledge management, including planning, retrieval, summarization, and context updating. We reproduce M-ASK using the official codebase.\footnote{\url{https://github.com/chenyiqun/M-ASK}} We train for up to $500$ steps, saving checkpoints every $20$ steps and using the last stable checkpoint if training collapses.
\end{itemize}

\section{Additional results and analysis}

\subsection{Main results evaluated by LLM judgement}
\label{app:llm_judge_results}

Since answer EM may fail to capture semantically correct responses that do not exactly match the reference answer, we additionally evaluate the main results using LLM-based judgment with \texttt{gpt-oss-120b} as the judge model. We use reasoning effort \texttt{medium}, and provide the full judging prompt in Appendix~\ref{app:prompts}. The results are reported in \Cref{tab:llm_judge_results} and show the same overall trend as the EM evaluation, with DAC outperforming all baselines overall.

\begin{table}[H]
\centering
\caption{\textbf{LLM judgement on QA benchmarks.} We use \texttt{gpt-oss-120b} as the judge. Datasets marked with $^\ast$ are in-domain datasets. \textbf{Bold} and \underline{underlined} numbers indicate the best and second best results for each model, respectively.
}
\label{tab:llm_judge_results}
\setlength{\tabcolsep}{4pt}
\resizebox{\textwidth}{!}{%
\begin{tabular}{llcccccccc}
\toprule
\multirow{2.5}{*}{Type} & \multirow{2.5}{*}{Method}
& \multicolumn{3}{c}{General QA} & \multicolumn{4}{c}{Multi-Hop QA} & \multirow{2.5}{*}{Average} \\
\cmidrule(lr){3-5} \cmidrule(lr){6-9}
& & NQ{$^*$} & TriviaQA & PopQA & HotpotQA\!{$^*$} & 2Wiki & MuSiQue & Bamb. & \\

\specialrule{\lightrulewidth}{\aboverulesep}{0pt}
\rowcolor[HTML]{F9ECEC}
\multicolumn{10}{c}{\rule[-1ex]{0pt}{3.4ex}\texttt{Qwen2.5-7B-Instruct}} \\
\specialrule{\lightrulewidth}{0pt}{\belowrulesep}
\multirow{5}{*}{\makecell[l]{Training- \\ free}} & Direct & 30.5 & 52.3 & 19.2 & 28.0 & 25.8 & 7.9 & 11.2 & 25.0 \\
& CoT & 33.7 & 56.2 & 19.0 & 30.5 & 22.8 & 10.2 & 30.4 & 29.0 \\
& Naive RAG & 55.4 & 71.4 & 46.8 & 42.7 & 21.2 & 9.9 & 28.0 & 39.4 \\
& IRCoT & 46.8 & 59.4 & 45.6 & 40.9 & 31.1 & 11.1 & 26.4 & 37.3 \\
& Search-o1 & 51.8 & 65.2 & 41.1 & 44.8 & 32.9 & 18.5 & 41.6 & 42.3 \\
\midrule
\multirow{2}{*}{Single-agent} & Search-R1 & 57.2 & 72.0 & 47.5 & 54.3 & 39.5 & 21.8 & \underline{50.4} & 49.0 \\
& TIPS & 57.7 & 73.8 & \underline{48.1} & 53.9 & 42.1 & 20.4 & 45.6 & 48.8 \\
\midrule
\multirow{3}{*}{Multi-agent} & S3 & \textbf{64.5} & \underline{74.7} & 47.7 & 55.6 & 36.4 & 19.4 & 44.0 & 48.9 \\
& DART & 56.0 & 72.0 & 45.7 & \underline{56.1} & \textbf{45.6} & \underline{22.1} & 49.6 & \underline{49.6} \\
& M-ASK & 58.6 & 74.6 & 45.5 & 55.2 & 40.4 & 21.6 & 44.0 & 48.6 \\
\midrule
Multi-agent & \textbf{DAC} & \underline{60.2} & \textbf{74.8} & \textbf{49.2} & \textbf{59.9} & \underline{45.2} & \textbf{25.5} & \textbf{54.4} & \textbf{52.7} \\

\specialrule{\lightrulewidth}{\aboverulesep}{0pt}
\rowcolor[HTML]{F9ECEC}
\multicolumn{10}{c}{\rule[-1ex]{0pt}{3.4ex}\texttt{Qwen3-8B (Non-Thinking)}} \\
\specialrule{\lightrulewidth}{0pt}{\belowrulesep}
\multirow{5}{*}{\makecell[l]{Training- \\ free}} & Direct & 32.5 & 52.5 & 20.5 & 28.1 & 27.1 & 7.8 & 12.8 & 25.9 \\
& CoT & 36.0 & 53.8 & 19.9 & 31.3 & 27.6 & 10.9 & 40.8 & 31.5 \\
& Naive RAG & 54.9 & 70.3 & 45.5 & 42.9 & 22.7 & 11.0 & 34.4 & 40.2 \\
& IRCoT & 57.9 & 74.5 & 50.0 & 53.5 & 34.4 & 17.7 & 47.2 & 47.9 \\
& Search-o1 & 56.3 & 72.1 & 45.3 & 51.3 & 40.9 & 22.1 & 48.0 & 48.0 \\
\midrule
\multirow{2}{*}{Single-agent} & Search-R1 & 61.8 & 77.1 & 50.4 & 62.3 & 50.1 & \textbf{29.6} & \underline{59.2} & \underline{55.8} \\
& TIPS & 59.4 & 74.5 & 49.1 & 55.7 & 47.7 & 26.4 & 56.0 & 52.7 \\ 
\midrule
\multirow{3}{*}{Multi-agent} & S3 & \underline{62.0} & \underline{78.0} & \underline{51.9} & 53.5 & 36.4 & 17.9 & 41.6 & 48.8 \\
& DART & 60.2 & 77.1 & \underline{51.9} & \underline{62.4} & \underline{50.7} & \underline{27.9} & \textbf{60.0} & 55.7 \\
& M-ASK & 61.3 & 77.9 & 48.1 & 62.0 & 49.1 & 26.2 & 50.4 & 53.6 \\
\midrule
Multi-agent & \textbf{DAC} & \textbf{63.5} & \textbf{78.8} & \textbf{52.5} & \textbf{64.2} & \textbf{54.1} & \textbf{29.6} & \underline{59.2} & \textbf{57.4} \\
\bottomrule
\end{tabular}%
}
\end{table}

\subsection{Comparison with DART under matched training conditions}
\label{app:dart_matched_results}

DART relies on an SFT warm-up stage and special search tag tokens as part of its original training pipeline. To control for these differences, we additionally compare DART and DAC under matched conditions, following the training order specified in the DART codebase. Specifically, we add the tag tokens \texttt{<search>} and \texttt{</search>} to the tokenizer, perform a lightweight SFT warm-up on 30K self-distilled trajectories, and then train both methods using two LoRA adapters with rank $32$ and $\alpha=64$. The self-distilled trajectories were sampled from the initial model and filtered by final-answer EM correctness. The results are reported in \Cref{tab:dart_matched_results}. DAC continues to outperform DART under this matched setup, suggesting that the gains are not primarily due to differences in initialization, tokenizer design, or adaptation capacity, but instead arise from DAC's training design.

\begin{table}[H]
\centering
\caption{\textbf{Matched comparison with DART.} We compare DAC with DART under matched conditions (SFT warmup, two LoRA adapters). \textbf{Bold} numbers indicate the best results for each model.}
\label{tab:dart_matched_results}
\setlength{\tabcolsep}{2.5pt}
\resizebox{\textwidth}{!}{%
\begin{tabular}{lcccccccccccccccc}
\toprule
\multirow{4}{*}{Method} & \multicolumn{6}{c}{General QA} & \multicolumn{8}{c}{Multi-Hop QA} & \multicolumn{2}{c}{\multirow{2.5}{*}{Average}} \\
\cmidrule(lr){2-7} \cmidrule(lr){5-15}
& \multicolumn{2}{c}{NQ{$^*$}} & \multicolumn{2}{c}{TriviaQA} & \multicolumn{2}{c}{PopQA} & \multicolumn{2}{c}{HotpotQA\!{$^*$}} & \multicolumn{2}{c}{2Wiki} & \multicolumn{2}{c}{MuSiQue} & \multicolumn{2}{c}{Bamb.} & \\
\cmidrule(lr){2-3} \cmidrule(lr){4-5} \cmidrule(lr){6-7} \cmidrule(lr){8-9} \cmidrule(lr){10-11} \cmidrule(lr){12-13} \cmidrule(lr){14-15} \cmidrule(lr){16-17}
& EM & LLM & EM & LLM & EM & LLM & EM & LLM & EM & LLM & EM & LLM & EM & LLM & EM & LLM \\

\specialrule{\lightrulewidth}{\aboverulesep}{0pt}
\rowcolor[HTML]{F9ECEC}
\multicolumn{17}{c}{\rule[-1ex]{0pt}{3.4ex}\texttt{Qwen2.5-7B-Instruct}} \\
\specialrule{\lightrulewidth}{0pt}{\belowrulesep}
DART & 47.1 & 60.0 & 64.8 & 75.7 & \textbf{48.7} & \textbf{53.3} & 40.9 & 60.1 & 43.7 & 45.2 & 18.1 & 26.1 & 37.6 & 48.0 & 43.0 & 52.6 \\
\textbf{DAC} & \textbf{48.8} & \textbf{61.6} & \textbf{65.5} & \textbf{76.7} & 46.2 & 51.0 & \textbf{45.7} & \textbf{63.1} & \textbf{43.9} & \textbf{50.0} & \textbf{19.5} & \textbf{27.0} & \textbf{41.6} & \textbf{49.6} & \textbf{44.5} & \textbf{54.1} \\

\specialrule{\lightrulewidth}{\aboverulesep}{0pt}
\rowcolor[HTML]{F9ECEC}
\multicolumn{17}{c}{\rule[-1ex]{0pt}{3.4ex}\texttt{Qwen3-8B (Non-Thinking)}} \\
\specialrule{\lightrulewidth}{0pt}{\belowrulesep}
DART & 44.4 & 58.8 & 64.9 & 76.7 & \textbf{48.1} & \textbf{53.9} & 47.0 & 64.3 & 41.4 & 47.9 & 22.4 & 30.0 & 45.6 & 53.6 & 45.2 & 55.0 \\
\textbf{DAC} & \textbf{48.7} & \textbf{60.3} & \textbf{65.4} & \textbf{77.0} & 47.8 & 51.5 & \textbf{47.3} & \textbf{65.5} & \textbf{44.7} & \textbf{51.0} & \textbf{23.0} & \textbf{30.4} & \textbf{45.6} & \textbf{54.4} & \textbf{46.1} & \textbf{55.7} \\
\bottomrule
\end{tabular}%
}
\end{table}

\subsection{Full Search-R1 reproduced results}
\label{app:search_r1_full}

Search-R1 has been trained under several configurations \citep{searchr1,searchr1_2}, so we evaluate four variants: PPO or GRPO, each with outcome-only or outcome and format rewards (\Cref{tab:search_r1_full}). Throughout the paper we report the variant with the highest average EM for each backbone: GRPO with outcome-only reward for \texttt{Qwen2.5-7B-Instruct}, GRPO with outcome and format rewards for \texttt{Qwen3-8B}. DAC exceeds all four variants in average EM and LLM-judged accuracy on both backbones.

\begin{table}[H]
\centering
\caption{\textbf{Search-R1 variants.} We evaluate all four Search-R1 configurations, combining PPO or GRPO with outcome-only (O) or outcome and format (OF) rewards, alongside DAC. \textbf{Bold} and \underline{underlined} numbers indicate the best and second best results for each model, respectively.}
\label{tab:search_r1_full}
\setlength{\tabcolsep}{2.5pt}
\resizebox{\textwidth}{!}{%
\begin{tabular}{lcccccccccccccccc}
\toprule
\multirow{4}{*}{Config} & \multicolumn{6}{c}{General QA} & \multicolumn{8}{c}{Multi-Hop QA} & \multicolumn{2}{c}{\multirow{2.5}{*}{Average}} \\
\cmidrule(lr){2-7} \cmidrule(lr){5-15}
& \multicolumn{2}{c}{NQ{$^*$}} & \multicolumn{2}{c}{TriviaQA} & \multicolumn{2}{c}{PopQA} & \multicolumn{2}{c}{HotpotQA\!{$^*$}} & \multicolumn{2}{c}{2Wiki} & \multicolumn{2}{c}{MuSiQue} & \multicolumn{2}{c}{Bamb.} & \\
\cmidrule(lr){2-3} \cmidrule(lr){4-5} \cmidrule(lr){6-7} \cmidrule(lr){8-9} \cmidrule(lr){10-11} \cmidrule(lr){12-13} \cmidrule(lr){14-15} \cmidrule(lr){16-17}
& EM & LLM & EM & LLM & EM & LLM & EM & LLM & EM & LLM & EM & LLM & EM & LLM & EM & LLM \\

\specialrule{\lightrulewidth}{\aboverulesep}{0pt}
\rowcolor[HTML]{F9ECEC}
\multicolumn{17}{c}{\rule[-1ex]{0pt}{3.4ex}\texttt{Qwen2.5-7B-Instruct}} \\
\specialrule{\lightrulewidth}{0pt}{\belowrulesep}
PPO \& O & 39.2 & 55.9 & 61.7 & \underline{73.1} & 39.5 & 45.9 & 38.0 & 54.0 & 34.4 & 40.0 & 15.6 & 21.6 & \underline{40.8} & \underline{53.6} & 38.5 & \underline{49.2} \\
PPO \& OF & 37.6 & 54.6 & 60.8 & 71.8 & 37.9 & 44.6 & 35.5 & 51.2 & 32.9 & 38.8 & 12.7 & 19.6 & 37.6 & 48.0 & 36.4 & 46.9 \\
GRPO \& O & \textbf{43.9} & 57.2 & \underline{62.0} & 72.0 & \underline{42.6} & \underline{47.5} & \underline{38.5} & \underline{54.3} & \underline{34.8} & 39.5 & 15.6 & 21.8 & 40.0 & 50.4 & \underline{39.6} & 49.0 \\
GRPO \& OF & \underline{39.6} & \underline{57.9} & 60.9 & 72.9 & 39.9 & 46.5 & 37.1 & 53.0 & 34.7 & \underline{40.5} & \underline{15.8} & \underline{23.1} & 35.2 & 45.6 & 37.6 & 48.5 \\
\midrule
DAC & \textbf{43.9} & \textbf{60.2} & \textbf{63.4} & \textbf{74.8} & \textbf{43.8} & \textbf{49.2} & \textbf{43.2} & \textbf{59.9} & \textbf{39.8} & \textbf{45.2} & \textbf{19.0} & \textbf{25.5} & \textbf{43.2} & \textbf{54.4} & \textbf{42.3} &\textbf{52.7} \\

\specialrule{\lightrulewidth}{\aboverulesep}{0pt}
\rowcolor[HTML]{F9ECEC}
\multicolumn{17}{c}{\rule[-1ex]{0pt}{3.4ex}\texttt{Qwen3-8B (Non-Thinking)}} \\
\specialrule{\lightrulewidth}{0pt}{\belowrulesep}
PPO \& O & 43.4 & 56.5 & 60.5 & 70.6 & 43.4 & 48.0 & 40.3 & 55.6 & 38.2 & 43.6 & 16.5 & 22.0 & 48.0 & \underline{58.4} & 41.5 & 50.7 \\
PPO \& OF & 45.1 & 58.5 & 61.8 & 73.1 & 45.8 & 50.4 & 42.2 & 58.3 & 42.2 & 47.9 & 19.4 & 26.2 & \textbf{49.6} & \textbf{59.2} & 43.7 & 53.4 \\
GRPO \& O & \textbf{48.8} & \underline{62.2} & 66.0 & 76.6 & \underline{46.9} & \underline{51.6} & 44.8 & 61.0 & 40.5 & 46.3 & 18.3 & 26.1 & \underline{48.8} & 56.8 & 44.9 & 54.4 \\
GRPO \& OF & 48.1 & 61.8 & \underline{66.1} & \underline{77.1} & 45.3 & 50.4 & \underline{45.0} & \underline{62.3} & \underline{43.9} & \underline{50.1} & \underline{21.7} & \underline{29.6} & \underline{48.8} & \textbf{59.2} & \underline{45.6} & \underline{55.8} \\
\midrule
DAC & \underline{48.4} & \textbf{63.5} & \textbf{67.2} & \textbf{78.8} & \textbf{47.9} & \textbf{52.5} & \textbf{46.6} & \textbf{64.2} & \textbf{48.7} & \textbf{54.1} & \textbf{22.3} & \textbf{29.6} & \textbf{49.6} & \textbf{59.2} & \textbf{47.2} & \textbf{57.4} \\
\bottomrule
\end{tabular}%
}
\end{table}

\subsection{Full results for disentangled search and generation}
\label{app:full_disentangled_results}

\begin{table}[t]
\centering
\caption{\textbf{Search accuracy by benchmark.} Search accuracy of retrieved evidence for Search-R1, DART, and DAC, judged by \texttt{gpt-oss-120b}. \textbf{Bold} numbers indicate the best results for each model.
}
\label{tab:search_accuracy}
\resizebox{\textwidth}{!}{%
\begin{tabular}{lcccccccc}
\toprule
\multirow{2.5}{*}{Method} & \multicolumn{3}{c}{General QA} & \multicolumn{4}{c}{Multi-Hop QA} & \multirow{2.5}{*}{Average} \\
\cmidrule(lr){2-4} \cmidrule(lr){5-8}
& NQ{$^*$} & TriviaQA & PopQA & HotpotQA\!{$^*$} & 2Wiki & MuSiQue & Bamb. & \\

\specialrule{\lightrulewidth}{\aboverulesep}{0pt}
\rowcolor[HTML]{F9ECEC}
\multicolumn{9}{c}{\rule[-1ex]{0pt}{3.4ex}\texttt{Qwen2.5-7B-Instruct}} \\
\specialrule{\lightrulewidth}{0pt}{\belowrulesep}
Search-R1 & 69.7 & 75.4 & 55.1 & 58.7 & 36.5 & 30.7 & 54.4 & 54.3 \\
DART & 67.2 & 74.9 & 52.1 & 59.9 & \textbf{42.3} & 28.7 & 48.8 & 53.4 \\
\midrule
\textbf{DAC} & \textbf{73.5} & \textbf{78.5} & \textbf{56.3} & \textbf{66.1} & 42.1 & \textbf{34.1} & \textbf{56.8} & \textbf{58.2} \\

\specialrule{\lightrulewidth}{\aboverulesep}{0pt}
\rowcolor[HTML]{F9ECEC}
\multicolumn{9}{c}{\rule[-1ex]{0pt}{3.4ex}\texttt{Qwen3-8B (Non-Thinking)}} \\
\specialrule{\lightrulewidth}{0pt}{\belowrulesep}
Search-R1 & 71.7 & 76.8 & 54.5 & 63.9 & 42.1 & 36.0 & \textbf{61.6} & 58.1 \\
DART & 71.3 & 77.2 & 55.2 & 63.0 & 42.9 & 36.4 & 60.0 & 58.0 \\
\midrule
\textbf{DAC} & \textbf{74.5} & \textbf{79.8} & \textbf{56.7} & \textbf{67.1} & \textbf{44.0} & \textbf{37.7} & \textbf{61.6} & \textbf{60.2} \\
\bottomrule
\end{tabular}%
}
\end{table}

\begin{table}[t]
\vspace{-5pt}
\centering
\caption{\textbf{Answer accuracy under fixed retrieval by benchmark.} Answer EM when the baseline or DAC generator uses the same retrieval results produced by Search-R1 or DART. \textbf{Bold} numbers indicate the best results in each block.
}
\label{tab:fixed_retrieval_answer_accuracy}
\resizebox{\textwidth}{!}{%
\begin{tabular}{lcccccccc}
\toprule
\multirow{2.5}{*}{Method} & \multicolumn{3}{c}{General QA} & \multicolumn{4}{c}{Multi-Hop QA} & \multirow{2.5}{*}{Average} \\
\cmidrule(lr){2-4} \cmidrule(lr){5-8}
& NQ{$^*$} & TriviaQA & PopQA & HotpotQA\!{$^*$} & 2Wiki & MuSiQue & Bamb. & \\

\specialrule{\lightrulewidth}{\aboverulesep}{0pt}
\rowcolor[HTML]{F9ECEC}
\multicolumn{9}{c}{\rule[-1ex]{0pt}{3.4ex}\texttt{Qwen2.5-7B-Instruct}} \\
\specialrule{\lightrulewidth}{0pt}{\belowrulesep}
Search-R1 & \textbf{43.9} & 62.0 & 42.6 & 38.5 & 34.8 & 15.6 & 40.0 & 39.6 \\
\textbf{DAC} & 43.4 & \textbf{63.2} & \textbf{43.2} & \textbf{41.1} & \textbf{37.4} & \textbf{16.3} & \textbf{42.4} & \textbf{41.0} \\
\midrule
DART & 42.4 & 61.6 & 41.0 & 39.9 & \textbf{40.6} & \textbf{15.7} & 40.8 & 40.3 \\
\textbf{DAC} & \textbf{42.8} & \textbf{61.8} & \textbf{41.6} & \textbf{40.8} & 39.7 & 15.2 & \textbf{42.4} & \textbf{40.6} \\

\specialrule{\lightrulewidth}{\aboverulesep}{0pt}
\rowcolor[HTML]{F9ECEC}
\multicolumn{9}{c}{\rule[-1ex]{0pt}{3.4ex}\texttt{Qwen3-8B (Non-Thinking)}} \\
\specialrule{\lightrulewidth}{0pt}{\belowrulesep}
Search-R1 & 48.1 & 66.1 & 45.3 & 45.0 & 43.9 & \textbf{21.7} & 48.8 & 45.6 \\
\textbf{DAC} & \textbf{48.8} & \textbf{66.2} & \textbf{46.5} & \textbf{45.7} & \textbf{46.7} & 21.5 & \textbf{52.0} & \textbf{46.8} \\
\midrule
DART & 46.8 & 66.3 & \textbf{47.3} & 45.3 & 44.5 & 20.9 & 48.0 & 45.6 \\
\textbf{DAC} & \textbf{48.6} & \textbf{66.5} & \textbf{47.3} & \textbf{46.0} & \textbf{47.3} & \textbf{21.8} & \textbf{48.8} & \textbf{46.6} \\
\bottomrule
\end{tabular}%
}
\vspace{-5pt}
\end{table}

We provide the per-benchmark results underlying \Cref{fig:disentangled_results}. For search accuracy, we use LLM-based verification with \texttt{gpt-oss-120b} (reasoning effort: \texttt{medium}), which provides a more reliable measure than substring EM and avoids directly evaluating the baselines with the same lightweight verifier used during DAC training. The resulting search accuracy scores are reported in \Cref{tab:search_accuracy}. We further report answer EM when DAC generates from fixed retrieval results produced by Search-R1 or DART in \Cref{tab:fixed_retrieval_answer_accuracy}, isolating generation quality from retrieval quality. As discussed in \cref{sec:results}, DAC improves both search and generation quality.

\subsection{Further analysis of the effect of cross-verification}
\label{app:ablation_cross_verification_further}

\subsubsection{Cross-checkpoint searcher--generator evaluation}
\label{app:cross_checkpoint_analysis}

\begin{figure}[tb]
    \centering
    \includegraphics[width=0.99\textwidth]{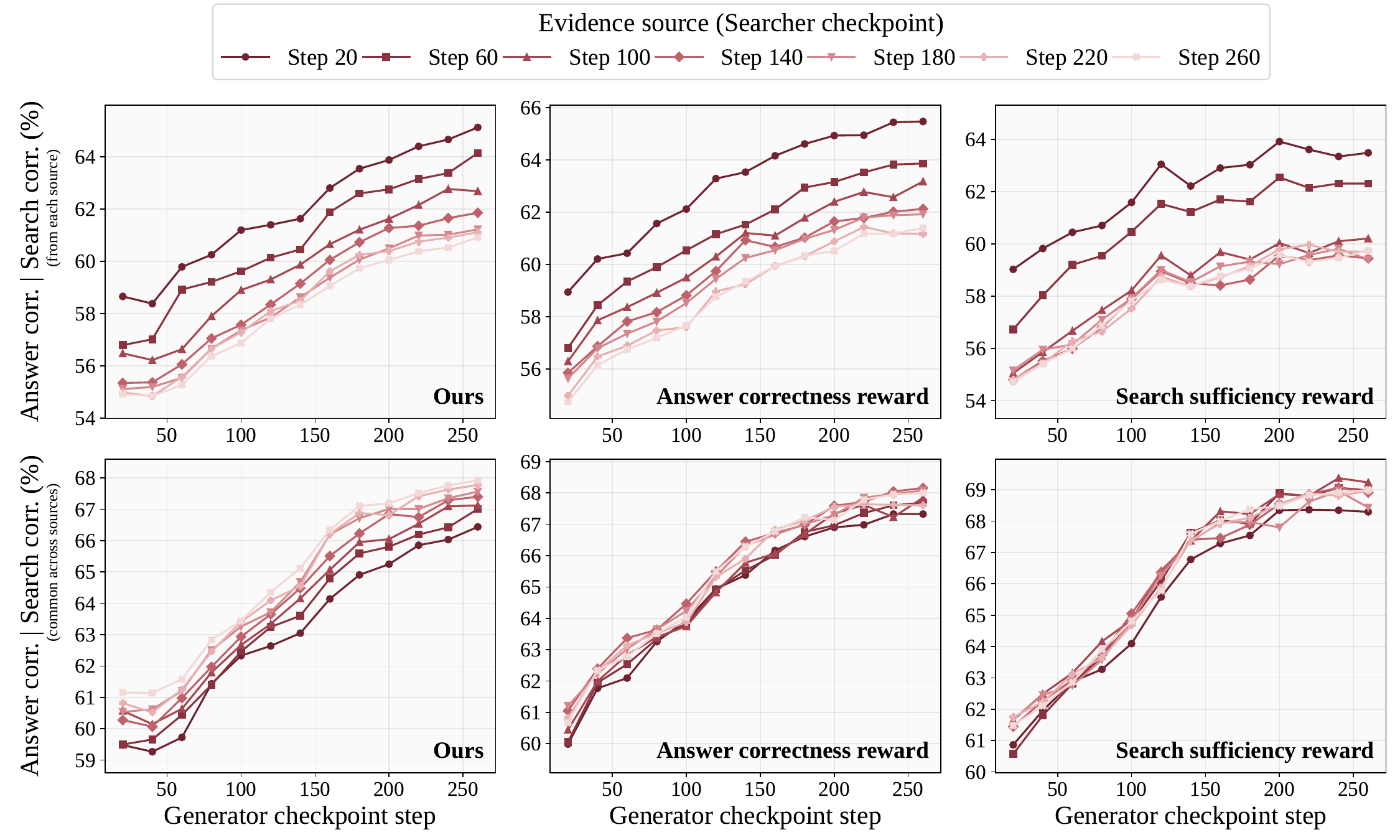}
    \caption{\textbf{Cross-checkpoint conditional answer accuracy.} Each column corresponds to a different searcher reward. We evaluate generator checkpoints on evidence produced by different searcher checkpoints, with abstention disabled. \emph{Top}: answer accuracy conditioned on the search-correct examples of each evidence source. \emph{Bottom}: the same evaluation restricted to the intersection of questions retrieved correctly by all evidence source checkpoints, keeping the evaluated question set fixed. Under DAC, later generators perform better on fixed evidence, while the apparent difficulty increase of later search evidence in the top row disappears when question composition is controlled.}
    \label{fig:conditional_accuracy_sweep}
\end{figure}

To further examine how the searcher and generator co-adapt during training, we evaluate generators from different checkpoints on evidence produced by searchers from different checkpoints. Abstention is disabled in this analysis, and we report answer accuracy conditioned on correct search. The top row of \Cref{fig:conditional_accuracy_sweep} conditions separately on questions where each searcher checkpoint successfully retrieves correct evidence, whereas the bottom row restricts evaluation to the intersection of questions that are retrieved correctly by all compared searcher checkpoints, thereby keeping the question set fixed. The three searcher reward variants follow the definitions in \cref{sec:ablation_cross_verification}.

Under DAC, the top row exhibits a clear monotonic, ordered structure. For a fixed evidence source, later generators achieve higher answer accuracy, confirming that \emph{the generator becomes better at exploiting the same retrieved evidence}. Conversely, for a fixed generator, evidence from later search checkpoints yields lower aggregate accuracy, yet the bottom row clarifies that this does not arise because later searchers degrade evidence on already-solved questions. When the question set is held fixed, later searchers instead yield comparable or higher answer accuracy. The aggregate decline therefore primarily reflects that \emph{later searchers retrieve correctly on additional, harder questions that earlier searchers fail to cover}. This is consistent with a curriculum in which retrieval progressively expands toward more difficult parts of the task as generation capability improves.

However, the searcher reward ablations show weaker forms of this structure. Answer correctness reward retains some cross-evidence monotonic ordering in the aggregate evaluation, consistent with the searcher still receiving a signal tied to downstream generation. However, this ordering largely disappears on the fixed-question subset, suggesting less room for continued improvement once the searcher already retrieves evidence from which the generator can answer correctly. Search sufficiency reward shows substantially weaker alignment across both evaluations, consistent with its searcher objective being misaligned to the evolving generator.

\subsubsection{How cross-verification reshapes search behavior after abstention}
\label{app:abstain_transition_analysis}

\begin{figure}[tb]
    \centering
    \includegraphics[width=0.99\textwidth]{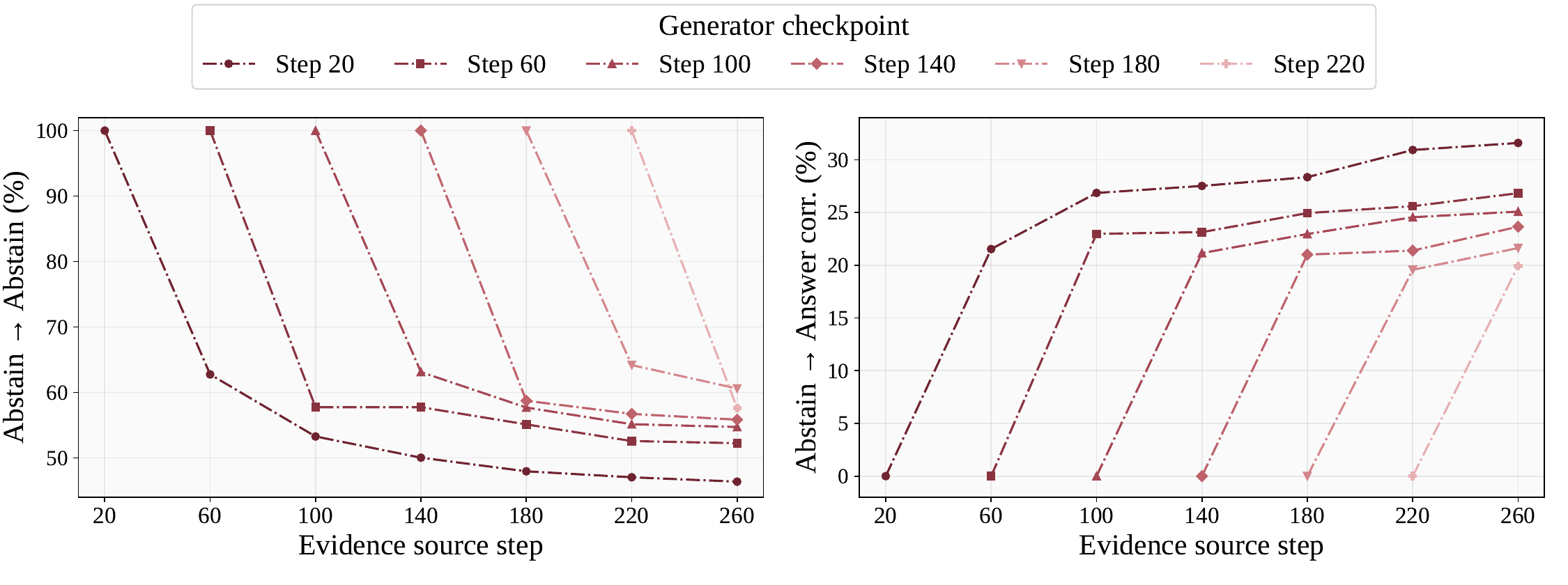}
    \caption{\textbf{Effect of later search evidence on previously abstained examples.} For each generator checkpoint, we select questions for which its matched searcher retrieves sufficient evidence but the generator abstains, then keep the generator fixed while replacing the evidence with that produced by later searcher checkpoints. \emph{Left}: fraction that remain abstained. \emph{Right}: fraction for which the fixed generator produces the correct answer. Later search evidence increasingly converts abstained cases into answerable ones.}
    \label{fig:abstain_transitions}
\end{figure}

We next focus on cases where, at a given checkpoint, the searcher retrieves sufficient evidence but the corresponding generator abstains. For each such checkpoint, we keep the generator fixed and replace its evidence with that produced by progressively later searchers on the same questions. As shown in \Cref{fig:abstain_transitions}, later searchers substantially reduce the probability that the fixed generator continues to abstain, while increasing the probability that it produces the correct answer. Thus, an abstention decision not only reflects the generator's current limitation but also provides a learning signal that pushes the searcher toward alternative sufficient evidence that the generator can more readily use. This illustrates how cross-verification can improve cases where retrieval is already sufficient but the generator cannot yet make use of the retrieved evidence.

\subsection{Full results for shared vs.\ dedicated policies}
\label{app:shared_dedicated_results}

\Cref{tab:shared_dedicated} reports the full per-benchmark comparison for \cref{sec:shared_dedicated}. Dedicated LoRA policies improve average EM on both backbones, from 38.5 to 39.8 on \texttt{Qwen2.5-7B-Instruct} and from 45.2 to 46.6 on \texttt{Qwen3-8B}, with gains on most benchmarks. This is consistent with the training dynamics in \Cref{fig:shared_dedicated}, where dedicated policies eventually outperform the shared policy.

\begin{table}[t]
\centering
\caption{\textbf{Shared vs. dedicated LoRA policies.} Answer EM for a shared policy or dedicated policies for the searcher and generator. \textbf{Bold} numbers indicate the best results in each block.
}
\label{tab:shared_dedicated}
\resizebox{\textwidth}{!}{%
\begin{tabular}{lcccccccc}
\toprule
\multirow{2.5}{*}{Method} & \multicolumn{3}{c}{General QA} & \multicolumn{4}{c}{Multi-Hop QA} & \multirow{2.5}{*}{Average} \\
\cmidrule(lr){2-4} \cmidrule(lr){5-8}
& NQ{$^*$} & TriviaQA & PopQA & HotpotQA\!{$^*$} & 2Wiki & MuSiQue & Bamb. & \\

\specialrule{\lightrulewidth}{\aboverulesep}{0pt}
\rowcolor[HTML]{F9ECEC}
\multicolumn{9}{c}{\rule[-1ex]{0pt}{3.4ex}\texttt{Qwen2.5-7B-Instruct}} \\
\specialrule{\lightrulewidth}{0pt}{\belowrulesep}
Shared policy & 39.6 & 60.2 & 38.9 & \textbf{40.7} & 34.1 & 16.3 & 40.0 & 38.5 \\
Dedicated policies & \textbf{41.9} & \textbf{61.2} & \textbf{41.1} & 40.3 & \textbf{35.4} & \textbf{17.2} & \textbf{41.6} & \textbf{39.8} \\

\specialrule{\lightrulewidth}{\aboverulesep}{0pt}
\rowcolor[HTML]{F9ECEC}
\multicolumn{9}{c}{\rule[-1ex]{0pt}{3.4ex}\texttt{Qwen3-8B (Non-Thinking)}} \\
\specialrule{\lightrulewidth}{0pt}{\belowrulesep}
Shared policy & 45.5 & 65.0 & 46.5 & 45.3 & 44.6 & \textbf{22.2} & 47.2 & 45.2 \\
Dedicated policies & \textbf{49.2} & \textbf{66.4} & \textbf{47.9} & \textbf{45.7} & \textbf{46.0} & 21.1 & \textbf{49.6} & \textbf{46.6} \\
\bottomrule
\end{tabular}%
}
\end{table}

\subsection{Accuracy when forcing abstained examples to answer}
\label{app:abstain_forced}

When the generator chooses to abstain and is forced to produce an answer, accuracy drops substantially compared to non-abstained examples. As shown in \Cref{tab:forced_answer}, for \texttt{Qwen3-8B}, the forced answer accuracy on abstained examples averages 10.94\% across benchmarks, compared to 53.84\% selective accuracy on non-abstained examples. This confirms that abstention decisions reflect genuine uncertainty about evidence sufficiency rather than arbitrary behavior.

\begin{table}[t]
\centering
\caption{\textbf{Answer accuracy by abstention decision.} \emph{Non-abstain} reports answer EM on examples the generator chooses to answer, while \emph{Abstain} reports forced-answer EM on examples it chooses to abstain from.}
\label{tab:forced_answer}
\begin{tabularx}{\textwidth}{l*{4}{>{\centering\arraybackslash}X}}
    \toprule
    \multirow{2.5}{*}{Benchmark} & \multicolumn{2}{c}{\texttt{Qwen2.5-7B-Instruct}} & \multicolumn{2}{c}{\texttt{Qwen3-8B (Non-Thinking)}} \\
    \cmidrule(lr){2-3} \cmidrule(lr){4-5}
    & Non-abstain & Abstain & Non-abstain & Abstain \\
    \midrule
    NQ                & 46.09\% & 13.13\% & 54.11\% & 11.68\% \\
    TriviaQA          & 67.77\% & 13.81\% & 72.00\% & 11.22\% \\
    PopQA             & 50.53\% & 14.42\% & 55.11\% & 12.90\% \\
    HotpotQA          & 49.57\% & 13.30\% & 52.58\% & 9.98\%  \\
    2WikiMultihopQA   & 47.52\% & 17.63\% & 59.14\% & 20.74\% \\
    MuSiQue           & 26.47\% & 4.90\%  & 27.86\% & 4.51\%  \\
    Bamboogle         & 46.08\% & 13.04\% & 56.07\% & 5.56\%  \\
    \midrule
    Average           & 47.72\% & 12.89\% & 53.84\% & 10.94\% \\
    \bottomrule
\end{tabularx}
\end{table}

\section{Prompt templates}
\label{app:prompts}

This section presents the prompts used in our various experiments.

Red-colored text in brackets indicates placeholders to be filled in. Both agents are prompted through a single user-role message. For readability, we present each prompt as an instruction followed by the runtime input fields.
The \textcolor{MainRed}{\texttt{[think tag]}} is set to \texttt{think} for \texttt{Qwen2.5-7B-Instruct} and \texttt{reason} for \texttt{Qwen3-8B}, since \texttt{<think>} is a reserved token in \texttt{Qwen3}'s native thinking mechanism and would conflict with the model's built-in thinking mode.

\Cref{prompt:qg,prompt:ag,prompt:ag_val} correspond to the prompts used for the \textbf{searcher agent}, the \textbf{generator agent} during \emph{training}, and the \textbf{generator agent} during \textbf{validation}, respectively.
The searcher agent is initialized with an instruction and the question. After each retrieval turn, the model's reasoning, search query, and retrieved documents are appended to the context as \texttt{<}\textcolor{MainRed}{\texttt{[think tag]}}\texttt{>...}\texttt{</}\textcolor{MainRed}{\texttt{[think tag]}}\texttt{>}\texttt{<search>...</search><information>...</information>} blocks.
The training and test prompts for the generator differ in that the training prompt instructs the model to output \texttt{unknown} when the answer cannot be found from the provided documents, enabling abstention, while the test prompt omits this instruction since abstention is disabled at test time.

\begin{prompt}[label=prompt:qg]{Searcher agent}
\setlength{\parskip}{8pt}
\textbf{Instruction:}

You are a research assistant that plans retrieval.
Your job is to collect information to answer the given question.
You must conduct reasoning inside \texttt{<}\textcolor{MainRed}{\texttt{[think tag]}}\texttt{>} and \texttt{</}\textcolor{MainRed}{\texttt{[think tag]}}\texttt{>} first every time you get new information.
After reasoning, if you find you lack some knowledge, you can call a search engine by \texttt{<search>} query \texttt{</search>} and it will return the top searched results between \texttt{<information>} and \texttt{</information>}.
You can repeat this process as many times as you want.
If you find no further external knowledge needed, simply output \texttt{<stop>}.

\vskip 10pt
\textbf{Runtime input:}

Question: \textcolor{MainRed}{\texttt{[Question]}} \\
\textcolor{MainRed}{\texttt{[Search history appended after each turn]}}
\end{prompt}

\begin{prompt}[label=prompt:ag]{Generator agent (Training)}
\setlength{\parskip}{8pt}
\textbf{Instruction:}

You are a knowledgeable assistant that answers questions.
You will be provided with a question and some relevant documents.
First, think step-by-step about how to answer the question based on the provided documents inside \texttt{<}\textcolor{MainRed}{\texttt{[think tag]}}\texttt{>} and \texttt{</}\textcolor{MainRed}{\texttt{[think tag]}}\texttt{>}.
Then, provide a concise and accurate answer inside \texttt{<answer>} and \texttt{</answer>}.
The answer should be in a short phrase or single word if possible.
If you cannot find the answer from the provided documents, simply answer \texttt{unknown} inside \texttt{<answer>} and \texttt{</answer>}.

\vskip 10pt
\textbf{Runtime input:}

Question: \textcolor{MainRed}{\texttt{[Question]}} \\
Documents: \\
\textcolor{MainRed}{\texttt{[Retrieved documents]}}
\end{prompt}

\begin{prompt}[label=prompt:ag_val]{Generator agent (Test)}
\setlength{\parskip}{8pt}
\textbf{Instruction:}

You are a knowledgeable assistant that answers questions.
You will be provided with a question and some relevant documents.
First, think step-by-step about how to answer the question based on the provided documents inside \texttt{<}\textcolor{MainRed}{\texttt{[think tag]}}\texttt{>} and \texttt{</}\textcolor{MainRed}{\texttt{[think tag]}}\texttt{>}.
Then, provide a concise and accurate answer inside \texttt{<answer>} and \texttt{</answer>}.
The answer should be in a short phrase or single word if possible.

\vskip 10pt
\textbf{Runtime input:}

Question: \textcolor{MainRed}{\texttt{[Question]}} \\
Documents: \\
\textcolor{MainRed}{\texttt{[Retrieved documents]}}
\end{prompt}

\Cref{prompt:llm_judge_answer,prompt:llm_judge_retrieval} correspond to the prompts used for LLM-based judgement, for answer accuracy and search accuracy, respectively. Since we prompt \texttt{gpt-oss-120b} with reasoning effort \texttt{medium}, we do not explicitly ask the judge to provide reasoning in the prompt.

\begin{prompt}[label=prompt:llm_judge_answer]{LLM-based judging for answer accuracy}
\setlength{\parskip}{8pt}
\textbf{Instruction:}
You will be given a question, answer, and gold answer triple.
Your task is to provide a yes/no answer indicating whether the answer is correct compared to one of the gold answers.
The answer is considered correct if it matches any of the gold answers within a reasonable degree of similarity.
However, the answer should be considered wrong if it is too broad, too narrow, or only partially correct compared to the gold answers.
Please answer with `yes' or `no' only.

\vskip 10pt
\textbf{Runtime input:}

Question: \textcolor{MainRed}{\texttt{[Question]}} \\
Answer: \textcolor{MainRed}{\texttt{[Generated answer]}} \\
Gold answers: \textcolor{MainRed}{\texttt{[Gold answers]}}
\end{prompt}

\begin{prompt}[label=prompt:llm_judge_retrieval]{LLM-based judging for search accuracy}
\setlength{\parskip}{8pt}
\textbf{Instruction:}
You will be given a question, some documents, and gold answer triple.
Your task is to provide a yes/no answer indicating whether the provided documents contain sufficient information to answer the question with one of the given gold answers.
Please answer with `yes' or `no' only.

\vskip 10pt
\textbf{Runtime input:}

Question: \textcolor{MainRed}{\texttt{[Question]}} \\
Documents: \textcolor{MainRed}{\texttt{[Retrieved documents]}} \\
Gold answers: \textcolor{MainRed}{\texttt{[Gold answers]}}
\end{prompt}

\end{document}